\relax
\documentclass[letterpaper]{article} 
\usepackage{times}  
\usepackage{helvet}  
\usepackage{courier}  
\usepackage{url}  
\usepackage{graphicx}  
\usepackage{rmthm}
\usepackage{graphics}
\usepackage{wrapfig}
\usepackage{rotating}

\usepackage[utf8]{inputenc}
\usepackage[small]{caption}

\usepackage[hide]{ed}

\usepackage{xifthen}
\usepackage{xspace} 
\usepackage[sans]{dsfont}
\usepackage{amsmath}
\usepackage{amssymb}
\usepackage[all]{xy}
\xyoption{v2}
\usepackage{stmaryrd}

\newtheorem{theorem}{Theorem}
\newtheorem{lemma}{Lemma}

\newtheorem{definition}{Definition}
\newtheorem{example}{Example}

\newtheorem{corollary}{Corollary}

\newenvironment{tableenv}
    {\begin{center}
	\begin{scriptsize}

    }
    { 

	\end{scriptsize}
    \end{center}
    }

\newcommand{\qed} {\hfill{$\Box$}}
\newenvironment{proof}[1]{\noindent{\bf Proof#1: }}{\qed\medskip}

\newcommand{\eat}[1]{}

\newcommand{\ind}{\mathit{indegree}}
\newcommand{\supps}{\mathit{supports}}
\newcommand{\atts}{\mathit{attacks}}

\newcommand{\sgn}{\mathit{sgn}}
\newcommand{\abs}{\mathit{abs}}

\newcommand{\arggraph}{\ensuremath{\mathds{A}}\xspace}

\newcommand{\argset}{\ensuremath{\mathcal{A}}\xspace}

\newcommand{\bs}{\ensuremath{\texttt{S}}\xspace}

\newcommand{\D}{\mathbb{D}}
\newcommand{\R}{\mathbb{R}}
\newcommand{\N}{\mathbb{N}}

\newcommand{\accdegr}[3]%
	{\ensuremath{\texttt{Deg}
	^{\ifthenelse{\equal{#1}{}}{\bs}{#1}}%
	_{\ifthenelse{\equal{#2}{}}{\arggraph}{#2}}
	{\ifthenelse{\equal{#3}{}}{ }{(#3)}}
	\xspace}}
\newcommand{\accdegrvec}[2]%
	{\ensuremath{\texttt{Deg}
	^{\ifthenelse{\equal{#1}{}}{\bs}{#1}}%
	_{\ifthenelse{\equal{#2}{}}{\arggraph}{#2}}
	\xspace}}
\newcommand{\accdegrvecnew}[2]%
	{\ensuremath{\texttt{Deg}
	^{\ifthenelse{\equal{#1}{}}{\bs}{#1}}%
	_{\ifthenelse{\equal{#2}{}}{\langle {G, w}\rangle}{#2}}
	\xspace}}

\newcommand{\maxd}{\ensuremath{\texttt{Max}_\texttt{S}}\xspace}
\newcommand{\mind}{\ensuremath{\texttt{Min}_\texttt{S}}\xspace}
\newcommand{\neutrald}{\ensuremath{\texttt{Neutral}_\texttt{S}}\xspace}	

\newcommand{\defaultarggraph}{$\arggraph  =\langle {\argset, G, w}\rangle $\xspace}
\newcommand{\alternativearggraph}{	$\arggraph ^\prime =\langle {\argset ^\prime, G^\prime, w^\prime}\rangle $\xspace}
\newcommand{\attackers}[2]{\ensuremath{\texttt{Att}_{\ifthenelse{\equal{#1}{}}
				{\arggraph}{#1}}({#2})}}
\newcommand{\defenders}[2]{\ensuremath{\texttt{Def}_{\ifthenelse{\equal{#1}{}}
				{\arggraph}{#1}}({#2})}}
\newcommand{\supporters}[2]{\ensuremath{\texttt{Sup}_{\ifthenelse{\equal{#1}{}}
				{\arggraph}{#1}}({#2})}}
\newcommand{\sattackers}[2]{\ensuremath{\texttt{sAtt}_{\ifthenelse{\equal{#1}{}}
				{\arggraph}{#1}}({#2})}}
\newcommand{\ssupporters}[2]{\ensuremath{\texttt{sSup}_{\ifthenelse{\equal{#1}{}}
				{\arggraph}{#1}}({#2})}}
\newcommand{\wasa}{\textsc{wasa}\xspace}

\newcommand{\backers}[2]{\ensuremath{\texttt{Back}_{\ifthenelse{\equal{#1}{}}
				{\arggraph}{#1}}({#2})}}
\newcommand{\detractors}[2]{\ensuremath{\texttt{Detr}_{\ifthenelse{\equal{#1}{}}
				{\arggraph}{#1}}({#2})}}
\newcommand{\parent}[2]{\ensuremath{\texttt{Parent}_{\ifthenelse{\equal{#1}{}}
				{\arggraph}{#1}}({#2})}}
\newcommand{\semdir}{\ensuremath{dir}}

\newcommand{\rownorm}[1]{{\interleave #1\interleave_\infty}}

\newcommand{\matr}[1]{\left(\begin{smallmatrix}#1\end{smallmatrix}\right)}

\newcommand{\twocase}[3]{
\left\{
\begin{array}{ll}
  #1,&\mbox{if }#2\\
  #3,&\mbox{otherwise}
\end{array}\right.}

\newcommand{\largefrac}[2]{\frac{\mathstrut#1}{\mathstrut#2}}

\newcommand{\aggr}{{\vec{\alpha}}}
\newcommand{\aggrp}{{\alpha}}
\newcommand{\inflp}{{\iota}}
\newcommand{\infl}{{\vec{\iota}}}
\newcommand{\charac}[1]{\textbf{#1}}
\usepackage{array}
\makeatletter
\newcommand{\thickhline}{%
    \noalign {\ifnum 0=`}\fi \hrule height 1pt
    \futurelet \reserved@a \@xhline
}
\newcolumntype{"}{@{\hskip\tabcolsep\vrule width 1pt\hskip\tabcolsep}}
\makeatother

\begin{document}


\pdfinfo{
/Title (Modular Semantics and Characteristics for Bipolar Weighted Argumentation Graphs)
/Author (Till Mossakowski, Fabian Neuhaus)}

\title{Modular Semantics and Characteristics for Bipolar Weighted Argumentation Graphs}

\author{Till Mossakowski \and Fabian Neuhaus\\
Otto-von-Guericke Universit\"at Magdeburg, Germany}

\maketitle

\begin{abstract}
This paper addresses the semantics of weighted argumentation graphs that are bipolar, i.e.\ 
contain both attacks and supports for arguments. It builds on previous work by 
Amgoud, Ben-Naim et. al.
\cite{DBLP:conf/kr/AmgoudB16,DBLP:conf/ijcai/AmgoudB16,DBLP:conf/ecsqaru/AmgoudB17,DBLP:conf/ijcai/AmgoudBDV17}. We study the various characteristics of acceptability semantics that
have been introduced in these works, and  introduce the notion 
of a \emph{modular acceptability semantics}.  A semantics is 
 modular if it cleanly separates \emph{aggregation} of
attacking and supporting arguments (for a given argument $a$) from the
computation of their \emph{influence} on $a$'s initial weight. 
We show that the
various semantics for bipolar argumentation graphs from the literature 
may be analysed as a composition of an aggregation function with an 
influence function. 
Based on this modular framework, we prove
general convergence and divergence theorems. We demonstrate that all well-behaved
modular acceptability semantics converge for all acyclic graphs and that no
\emph{sum}-based semantics can converge for all graphs. In particular,
we show divergence of Euler-based
semantics \cite{DBLP:conf/ecsqaru/AmgoudB17} for certain cyclic
graphs. Further,  we provide the first semantics for bipolar weighted
graphs that converges for all graphs.
\end{abstract}

\section{Introduction}

Abstract argumentation has been extensively studied since Dung's
pioneering work \cite{DBLP:journals/ai/Dung95} on argumentation graphs
featuring an attack relation between arguments. Dung's
framework has been generalised to
gradual (or rank-based, or weighted) argumentation graphs that assign
real numbers as weights to arguments (instead of just Boolean values),
as in \cite{DBLP:conf/kr/AmgoudB16,DBLP:conf/ijcai/AmgoudB16,DBLP:conf/ijcai/AmgoudBDV17}. These initial weights may be interpreted as representing the  acceptability of an argument on its own, that is without considering the effects of supports or attacks by other arguments. Initial weights 
may be implemented in different ways; e.g., in a previous work  we calculated them based on search results \cite{thiel2017web}, while in \cite{DBLP:conf/ijcai/LeiteM11} they are based on votes. 

In this paper we study bipolar \cite{DBLP:journals/ijar/CayrolL13}
weighted argumentation graphs that contain both attack and support
relationships, building in particular on the results of \cite{DBLP:conf/ecsqaru/AmgoudB17}.
One major research question in this area is the following: Given a collection of arguments that are connected by  attack and support relationships and their initial weights, what is the acceptability of the arguments? 
An acceptability semantics provides an answer to this question by providing a partial function that 
assigns acceptability degrees to arguments based on their initial weighs and the graph of the  
interactions between the arguments. Given that there are many possible acceptability semantics,  the challenge 
is to find acceptability semantics that are defined for most or even all argumentation graphs and 
display the expected characteristics of such a semantics. What characteristics an acceptability semantics should display is another research question. However, it is probably safe to state that 
some characteristics express expectations that are widely shared, while others express particular philosophical intuitions that won't be universally accepted or design choices that won't translate to other application contexts. Thus, it is unlikely that there will be the \emph{one} 
acceptability semantics for bipolar weighted argumentation graphs that everybody can agree on.
For this reason, we believe that is fruitful not to focus on individual acceptability semantics, but to study classes of acceptability semantics and characterise them in a modular way. This allows one to choose a particular acceptability semantics with the characteristics that are desirable in a given context.

The main contributions of this paper are the following: We show that the acceptability semantics for bipolar weighted argumentation graphs that 
have been studied in the literature  may be reformulated in a way that the  acceptability degree of an argument is calculated in two steps: firstly, an 
\emph{aggregation} function $\aggrp$ combines the effect of the predecessors of an argument; secondly, the 
\emph{influence} function $\inflp$ calculates the acceptability degree of the argument based on the result of the aggregation and the initial weights (see Fig.~\ref{fig:modular}). Acceptability semantics that follow this two-step approach, we call \emph{modular}. 
The fact that the acceptability semantics in the literature are all modular is no accident. Authors are interested in acceptability semantics that exhibit a number of desirable characteristics and the two-step approach of a modular acceptability semantics makes it easier to satisfy these characteristics.

Using a matrix representation of the acceptability semantics
(cf.\ e.g.\ \cite{DBLP:conf/comma/CoreaT16}), we propose a
mathematical elegant reformulation of the desirable characteristics of
modular acceptability semantics.  These characteristics provide
requirements on the aggregation function or the influence function
(see Table~\ref{tab:our-char}).  We distinguish between structural,
essential and optional characteristics, and show that some
characteristics from the literature are entailed by the structural and
essential characteristics.  While the reformulation of the
characteristics in a matrix notation may be a little inconvenient for readers
who are already familiar with the established notation in
\cite{DBLP:conf/kr/AmgoudB16,DBLP:conf/ijcai/AmgoudB16,DBLP:conf/ecsqaru/AmgoudB17,DBLP:conf/ijcai/AmgoudBDV17},
it has major benefits: the great elegance of Equation~(\ref{eq:rec1}),
secondly, the conciseness and ease of verification of characteristics,
also due to consistent application of a locality principle, and last
but not least, very general convergence and divergence results that we
can obtain, which crucially use matrix norms.

 \begin{figure}
\resizebox{\columnwidth}{!}{%
 {\small
 \xymatrixcolsep{1ex}
 \xymatrixrowsep{1ex}
 $$\xymatrix{
 \Text{attack/support\\structure}\ar[dr]\\
 &\Text{aggregation}\ar[r]<-0.5ex>&\Text{influence}\ar[r]<-0.5ex>&\Text{acceptability\\degree of\\ argument}\\
 \Text{acceptability\\degrees\\of parents}\ar[ur]
 &&\Text{initial weights}\ar[u]
 }$$
 }
 }
 \caption{Modular structure of acceptability semantics\label{fig:modular}}
 \vspace{-1.5em}
 \end{figure}
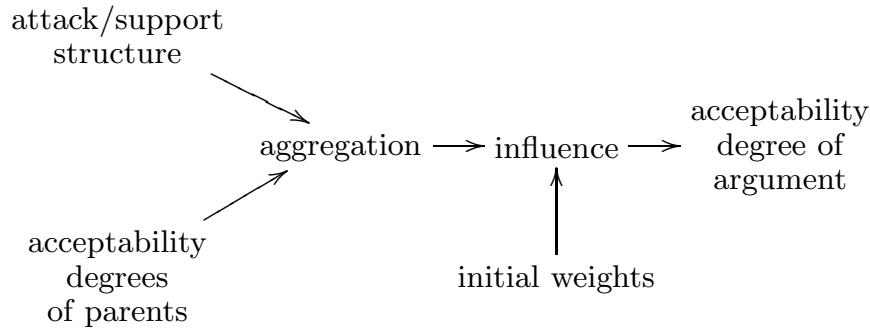
 Further, we provide an overview of five%
 \footnote{In addition, we discuss  two further functions
   from the non-bipolar literature,
   which satisfy the essential characteristics only if they
   are restricted to argumentation graphs with either only supports or only  attacks.} possible implementations of the aggregation 
function and five
\footnote{Further, we discuss  three additional  functions from the non-bipolar literature,
   which, however, are not usable for bipolar graphs in an unrestricted way.} possible implementations of the influence function, show that they meet the essential 
characteristics and discuss their optional characteristics (see Table~\ref{tab:aggr-impl}). 

While in the literature, convergence and divergence is studied for
each semantics separately, our modular approach allows us to obtain
convergence and divergence theorems for whole classes of semantics. 
E.g., all of the 25 possible combinations of {aggregation} functions and {influence} functions that we discuss in this paper  lead to semantics that are well-defined for 
acyclic graphs since they are well-behaved modular acceptability semantics (Theorem \ref{thm:limit-acyclic}).
However, none of the aggregation semantics from the literature is well-defined for all bipolar 
argumentation graphs. We propose 3 novel semantics that are well-defined for all bipolar argumentation 
graphs (see Table~\ref{tab:sem}) and characterise for several other semantics their convergence 
conditions.
Following our technical report \cite{DBLP:journals/corr/MossakowskiN16},
we recall some results about partially converging semantics.
Further, we show a systematic approach for proving divergence of acceptability semantics, which are based on the most popular implementation of the aggregation function.
Indeed, we show that no such acceptability semantics can
converge for all graphs.

The outline of the paper is as follows.  In section~\ref{sec:basics}
we introduce the basic notions.  In section~\ref{sec:characteristics}
we discuss the characteristics an acceptability semantics should have,
generalising and sometimes strengthening those of
\cite{DBLP:conf/kr/AmgoudB16,DBLP:conf/ijcai/AmgoudB16,DBLP:conf/ecsqaru/AmgoudB17,DBLP:conf/ijcai/AmgoudBDV17,DBLP:conf/aaai/BonzonDKM16}.
Section~\ref{sec:semantics-ag} is the central section of the paper.
It introduces our modular approach and applies it to several semantics
known from the literature, as well as to new ones. 
In section \ref{sec:comparison} we compare some selected 
modular acceptability semantics with the help of an example. It illustrates that 
different modular acceptability semantics  may support different philosophical intuitions and design choices, 
and, thus, may differ significantly in their evaluation of arguments.
Sections~\ref{sec:convergence} and~\ref{sec:divergence} prove general
convergence and divergence results, respectively. Section~\ref{sec:related}
discusses related work.
Finally, in
section~\ref{sec:future} we discuss some conclusion and future work.
Some proofs have been relegated to an appendix.

%

\section{Basic Concepts}\label{sec:basics}
In our definitions of weighting, argumentation graphs, and acceptability semantics we follow to a
 large degree \cite{DBLP:conf/ijcai/AmgoudB16,DBLP:conf/ijcai/AmgoudBDV17,DBLP:conf/ecsqaru/AmgoudB17}. One  difference of our approach is that
we allow weights of arguments to stem from any
connected subset $\D\subseteq\R$ 
 of the reals. Such a $\D$ is a parameter of our approach;
 it will serve as domain for argument
valuations. 
\cite{DBLP:conf/ijcai/AmgoudB16,DBLP:conf/ijcai/AmgoudBDV17}
assume that $\D=[0,1]$.
As in \cite{DBLP:conf/iberamia/BudanVS14} we assume that the attack
and the support relationships are disjoint, i.e.\ no argument
simultaneously attacks and supports another one.  We organise
argumentation graphs into incidence matrices with attack
represented by $-1$ \cite{DBLP:conf/comma/CoreaT16}.

We assume that $0\in\D$ is the \emph{neutral acceptability
  degree}. Attacks or supports by arguments
with the neutral acceptability degree will have no effect. Thus, arguments 
with neutral acceptability degree are semantically inert.  In \cite{DBLP:journals/corr/MossakowskiN16} we allowed $\neutrald$ to be any value in $\D$. 
In this paper we follow the convention in the literature to assume that 
$\neutrald = 0$. It has the benefit to simplify the presentation of some definitions and theorems.  

If
$\D$ has a minimum, we denote it by $\mind$, otherwise,
$\mind$ is undefined. Likewise with $\maxd$ for maximum.
Note that in
\cite{DBLP:conf/ijcai/AmgoudB16,DBLP:conf/ecsqaru/AmgoudB17,DBLP:conf/ijcai/AmgoudBDV17},
the minimum degree $\mind = 0$ (characterising a worthless argument, which
cannot be weakened further) is the same as the neutral degree
(characterising a semantically inert argument, which cannot influence
other arguments). We generalise this approach by allowing the
differentiation between minimum and neutral value. For example,
if $\D=[-1,1]$, then the minimum acceptability (or maximum unacceptability) is differentiated from the neutral acceptability degree.
Unacceptable arguments can be modelled using a negative degree.
The neutral value $0$
can be used for arguments that are right on the boundary between
acceptability and unacceptability. This is particularly
useful for initial weights, if we intend to represent an open mind
about the acceptability of an argument. %
 Note that an
attack on a neutral argument may turn it into an unacceptable
argument, and a support by an unacceptable argument may weakens
the acceptability degree of an argument.

\begin{definition}[Argumentation Graph]
	A weigh\-ted  attack/support argumentation graph (\wasa) 
	is a triple \defaultarggraph{}, 	where 
	\begin{itemize}
		\item \argset is a vector of size $n$ (where $n\in \mathbb{N}^+$), where all components of \argset are pairwise distinct; 
		\item $G=\{G_{ij}\}$ is a square matrix of order $n$ with $G_{ij}\in \{-1,0,1\}$, where  $G_{ij} = -1$ means that argument $a_j$ attacks $a_i$;  $G_{ij} = 1$ means that argument $a_j$ supports $a_i$; and $G_{ij} = 0$ if neither of these; 
		\item $w\in\D^n$ is a vector of initial weights.  
	\end{itemize}	
\end{definition}
$G_i$ denotes the $i$-th row of $G$ and represents the \emph{parents} (supporters
and attackers) of argument $a_i$. By abuse of notation, for $a=a_i$,
we will also write $G_a$ for $G_i$ and $w(a)$ for $w_i$.

\begin{example}\label{ex:arggraph1}

 $
 \arggraph^{ex1} = \left \langle 
 \left(\begin{smallmatrix}
 a_1 \\
 a_2 \\
 a_3 \\
 a_4\\
 a_5 \\
 a_6 
 \end{smallmatrix}\right), 
 \left(\begin{smallmatrix}
 0 & 0 &  0 &   0 &  0 &   0  \\ 
 0 & 0 &  0 &   0 &  0 &   0  \\ 
 -1 & 0 &  0 &   0 &  0 &   0  \\ 
 0 & -1 &  0 &   0 &  0 &   0  \\ 
 0 & 0 &  0 &   1 &  0 &  -1  \\ 
 0 & 0 &  0 &   0 &  -1 &   1   
 \end{smallmatrix}\right), 
 \left(\begin{smallmatrix}
0 \\
 1 \\
 0.4 \\
 0 \\
 0.5 \\
 1
 \end{smallmatrix}\right)
 \right \rangle 
 $	\quad	
 
\end{example}

\begin{wrapfigure}{R}{0.25\columnwidth}
 \centering 
 $\xymatrix{
 \underset{0}{a_1} \ar[d]  &
 \underset{1}{a_2}  \ar[d]  \\
 \underset{0.4}{a_3}   & 
 \underset{0}{a_4} \ar@{-->}[ld]   \\
\underset{0.5}{a_5}  \ar[r]<0.5ex> &
 \underset{1}{a_6} \ar[l]<0.5ex>   \ar@{-->}@(r,u)
 }$	
 \caption{Graphical representation of $ \arggraph^{ex1}$  \label{fig:example1}}
\end{wrapfigure}



The \wasa   $ \arggraph^{ex1}$  consists of the arguments $a_1, a_2, \ldots,  a_6$. 
$a_1$ and $a_2$ are neither attacked nor supported. $a_1$ attacks $a_3$ and $a_2$ attacks $a_4$, which in turn supports $a_5$. $a_5$ and $a_6$ attack each other, while $a_6$
supports itself. 

In the graphical representation the 
attack relationships are drawn as normal arrows and support relationships as dashed arrows (see Fig.~\ref{fig:example1}).

The initial weight of an argument in a \wasa represents its initial acceptability, that is the acceptability of the argument without consideration of its relation to other arguments. In the graphical representation we write the initial weights directly under the  names of the arguments (see Fig.~\ref{fig:example1}).
The semantics of the initial weights depends on the choice of $\D$. 
For example, if $\D = [-10, 10]$, there is not much difference between the initial acceptability of the arguments in $\arggraph^{ex1}$. In contrast, if $\D=[0,1]$, then   $a_1$ is valued as an argument with the lowest possible acceptability and $a_6$ is considered as a `perfect' argument. 

Given that 0 is the neutral acceptability degree, the attack of $a_1$  on $a_3$ in  $ \arggraph^{ex1}$ should have no effect, thus the acceptability degree of $a_3$ 
should equal its initial weight, namely $0.4$. In contrast, the effect of $a_2$ on $a_4$ depends on the choice of $\D$. As mentioned above, in \cite{DBLP:conf/ijcai/AmgoudB16,DBLP:conf/ecsqaru/AmgoudB17,DBLP:conf/ijcai/AmgoudBDV17},
 $\D=[0,1]$. In this context, $a_4$ is already valued as an argument with the worst possible acceptability degree. Thus, the attack of $a_2$ on $a_4$ cannot lower its acceptability any further. Thus, the acceptability degree of $a_4$ stays at 0. Since 0 is the neutral acceptability degree, the support of $a_4$ for $a_5$ has no effect on the acceptability degree of $a_5$.
 
However, $\neutrald = \mind = 0$ is not the only reasonable choice. Assume that the initial 
weights of arguments are determined based on upvotes and downvotes on a social media site, whose voters evaluate the merit of arguments. 
In this context, it is natural to represent acceptance and rejection 
in a symmetric way.
For example, one could calculate the initial value of an argument subtracting the number of downvotes from the number of upvotes and dividing it by the total of the votes. Thus, 
the initial weight 1 would represent that 100\% of voters accept the argument, the initial weight of -1 would represent  
 that 100\% of the voters reject the argument and $0$ would represent that acceptance and rejection are in balance. Hence $\D = [-1,1]$.
Under this interpretation $a_1$ and $a_4$ in $\arggraph^{ex1}$ are neither accepted nor rejected 
by the voting community, the community as a whole is undecided about the merits of   $a_1$ and $a_4$.  For this reason, as in the previous case, the attack of $a_1$ on $a_3$ would have no impact on the acceptability degree of $a_3$. However, in contrast to the previous case the strong attack of $a_2$ on $a_4$
would push the acceptability below 0, representing the rejection of $a_4$. One open question is whether the support of an unacceptable argument has no effect or should be counted against the supported argument. E.g., should the rejection of $a_4$ in $ \arggraph^{ex1}$  impact the acceptability degree of $a_5$. And if so, how much?

  %
 %

These questions are instances of a  more general question: Given a  \wasa, how do we calculate the acceptability of the arguments based on their initial weights and their relations? Following 
\cite{DBLP:conf/ijcai/AmgoudB16,DBLP:conf/ecsqaru/AmgoudB17,DBLP:conf/ijcai/AmgoudBDV17}, an 
answer to this question is called an acceptability semantics:

\begin{definition}[Acceptability Semantics] \label{def:semantics}
An acceptability semantics is a partial function \bs transforming any \wasa \defaultarggraph into a vector $\accdegr{}{}{}{}$ in $ \D^n$, where $n$ is the number of arguments in \arggraph.
For any argument $a_i$ in \argset, $(\accdegr{}{}{})_i$ (also noted as $\accdegr{}{}{a_i}$) is called the acceptability degree of $a_i$.
	\end{definition}
We define an acceptability semantics to be a \emph{partial} function
on argumentation graphs because some semantics only converge for a
subclass of graphs. For example, in
\cite{DBLP:conf/ecsqaru/AmgoudB17}, Euler-based semantics is
explicitly defined only for a subclass of graphs.

\begin{example}\label{def:matrix-exp-sem}
  One Example for an acceptability semantics is the matrix exponential semantics.
  It is straightforward to see that 
  matrix multiplication $G\cdot G$ computes all \emph{two-step}
  relations between arguments, where an attacker of an supporter
  is regarded as a two-step attacker, an attacker of a attacker
  is regarded as a two-step supporter etc.
  Similarly, $G^k$ computes all $k$-step ancestors (i.e.\ with
  path length $k$ from the ancestor to the argument).
  Based on this, the matrix exponential semantics  is defined as follows:
$$\accdegr{exp}{\langle {\argset, G, w}\rangle}{}{}=e^G\cdot w=(\sum_{k\in \mathbb{N}}\frac{G^k}{k!})w$$
computes the acceptability degree of an argument by summing up
initial weights of all ancestors of a given argument, weighted by
the factorial of the length of the path from the ancestor to the argument. See \cite{DBLP:conf/comma/CoreaT16}
for a use of matrix exponentials in a related but different context.


\end{example}


	

\section{Characteristics of Acceptability  Semantics
}	\label{sec:characteristics}



Obviously, there are many possible acceptability semantics, including
trivial (e.g. constant) ones. 
This raises the question
of desirable characteristics for acceptability semantics.

Table~\ref{tab:characteristics} provides an overview over most of the characteristics that were 
discussed for the semantics in 
\cite{DBLP:conf/ijcai/AmgoudB16,DBLP:conf/ecsqaru/AmgoudB17,DBLP:conf/ijcai/AmgoudBDV17}.%
\footnote{We omitted  \emph{Cardinality Preference}, \emph{Quality Preference} und \emph{Compensation}, which characterise alternative answers on the question whether the quality or the quantity of attacks (and supports, respectively) is more important.}
We have added  additional useful characteristics, namely \emph{modularity}, \emph{continuity}, 
 \emph{stickiness} and  \emph{symmetry}.
%
%
%
%
%
%
%
\begin{table*}
\begin{tableenv}

\begin{tabular}{|p{3.5cm}|p{3.5cm}|p{3.5cm}|p{3cm}|}\hline
	    Support Graphs in 		          \cite{DBLP:conf/ijcai/AmgoudB16}    & Attack Graphs in      \cite{DBLP:conf/ijcai/AmgoudBDV17} & Bipolar Graphs in \cite{DBLP:conf/ecsqaru/AmgoudB17}& Our terminology \\ \hline
	 Equivalence &  Equivalence  & Bi-variate Equivalence  &   Equivalence\\ 
	 -- & -- & -- &  Modularity\\ 
     Anonymity &  Anonymity  &   Anonymity  & Anonymity\\ 
	 Independence &  Independence & Bi-variate Independence &  Independence\\ 
	 	Reinforcement & 	 	Reinforcement  &   Bi-variate Reinforcement &  Reinforcement \\ 
	Coherence & Proportionality & -- & Initial monotonicity \\ 
	 Minimality &  Maximality  & Stability  & Stability\\ 
		-- & -- & -- & Continuity \\  %
	 Dummy &  Neutrality &Neutrality  &  Neutrality \\ 
	 Non-dilution &  Directionality &  Bi-variate Directionality &  Directionality\\ 
	 	Monotony & -- & Bi-variate 	Monotony  & Parent monotonicity\\ 
	Strengthening Soundness & Weakening Soundness & -- & Soundness\\ 
	Strengthening &  -- & 	Strengthening & Strengthening\\ 
			-- &  Weakening & Weakening & Weakening\\ 
	 Boundedness &  -- & -- & Compactness \\
	Imperfection & Resilience & Resilience & Resilience \\ %
		-- & -- & -- & Stickiness \\  %
			-- & -- & Franklin &  Franklin\\ %
	 Counting &  Counting & Bi-variate 	Monotony & Counting \\
-- 	 	 &  --  & --  & Symmetry \\
		\hline %
\end{tabular}

\smallskip

\caption{Characteristics of acceptability semantics 
 }
	\label{tab:characteristics}
\end{tableenv}
\vspace{-1.5em}
\end{table*}
%
%

While the intuitions behind these characteristics are stable, their formalisations vary. 
Some
  differences are mere technicalities that reflect that
  \cite{DBLP:conf/ijcai/AmgoudB16} is concerned only with support
  relationships, \cite{DBLP:conf/ijcai/AmgoudBDV17} is only about
  attack relationships, and \cite{DBLP:conf/ecsqaru/AmgoudB17} is
  about bipolar graphs. However, there are some major conceptual
  differences. For example, \emph{strengthening} and \emph{weakening}
  in \cite{DBLP:conf/ijcai/AmgoudB16,DBLP:conf/ijcai/AmgoudBDV17}
  differ very substantially from the corresponding concepts in
  \cite{DBLP:conf/ecsqaru/AmgoudB17}. 
  Some of characteristics in \cite{DBLP:conf/kr/AmgoudB16,DBLP:conf/ijcai/AmgoudB16,DBLP:conf/ecsqaru/AmgoudB17,DBLP:conf/ijcai/AmgoudBDV17}
  are formulated using two argumentation graphs, some
  (e.g.\ \emph{reinforcement}) one argumentation graph only. For
  \emph{reinforcement} (and other characteristics), we have found a
  version involving two argumentation graphs in
  \cite{DBLP:journals/corr/MossakowskiN16} which is stronger than the
  one-graph version but still follows the same intuition.
  Thus, there is no formulation of the characteristics that is in some sense canonical.

The characteristics in Table~\ref{tab:characteristics} vary in their nature. Some characteristics 
are \emph{essential} in the sense that any well-behaved acceptability semantics should exhibit them. For example, neutrality states that an argument with the neutral acceptability degree $0$ should not impact the acceptability degree of other arguments. 

Another kind of characteristics do not describe essential features of acceptability semantics, but provide a framework for the kind of  semantics one is considering. 
For example, anonymity implies that the acceptability degree of a given argument depends only on the structure of the argumentation graph and the initial weights. Hence, anonymity prohibits a semantics from considering other factors, for example  the internal structure of the arguments including their premises and consequences. 
This is clearly not an essential characteristic in the sense that all acceptability semantics should exhibit it. On the contrary, an acceptability semantics that takes the internal structure of arguments into account may lead to interesting results. Therefore, anonymity is not an essential nor even always desirable characteristic. Rather, anonymity expresses a choice concerning 
the scope of the semantics one intends to study. 
Moreover, when adopted, it simplifies the notation of the other characteristics, and therefore is a fundamental prerequisite for these.  These kind of characteristics we call \emph{structural}. 
In a sense, they set up
the framework in which the other characteristics are formulated.

A third kind of characteristics highlight interesting features of acceptability semantics. These are often the result of philosophical choices and their desirability may be contentions (e.g., 
stickiness or Franklin). These characteristics we call \emph{optional}.%
%
%
\footnote{
The distinction between structural, essential and optional characteristics is somewhat similar to the distinction between mandatory and optional axioms in \cite{DBLP:conf/ijcai/AmgoudB16}. However, note 
in  \cite{DBLP:conf/ijcai/AmgoudB16} both essential and structural characteristics would be considered as mandatory 
and
the
  top-based semantics in \cite{DBLP:conf/ijcai/AmgoudB16} does not  satisfy all the axioms which are said to be ``mandatory''.
}
%

In the following we formulate these characteristics in a way that fits
our matrix representation of the acceptability semantics and is
tailored to modular acceptability semantics. One goal of the
reformulation is to link each characteristic to conditions on the
aggregation function or the influence function. (Some characteristics
impose conditions on only one of the functions, others on both. See
Table~\ref{tab:our-char}.)  Another goal of the redesign was to
achieve more compact and mathematically elegant representations of the
characteristics, when possible.  Note that we sometimes have slightly strengthened the characteristics in a useful
way, while always ensuring that (1) they imply the characteristics of
\cite{DBLP:conf/ecsqaru/AmgoudB17}
(Thm.~\ref{thm:French-characteristics} below) and (2) they hold for the
semantics studied in the literature.

\subsection{Structural Characteristics}

\charac{Equivalence} requires that if two arguments start out with the
same initial weight and they share the same degree of attack and
support, they have the same acceptability degree. This is a locality
principle: the degree of an argument can explained in terms of the
degrees of its parents --- without knowing all degrees in the graph.
This reduces the cognitive complexity of understanding the resulting
acceptability degrees in essential way. This holds even in cyclic
graphs, where an argument can belong to its own ancestors, which of
course leads to a recursion. This recursion can be expressed as
follows, where $G_i$ is the $i$-th row of $G$:
\begin{equation}
  \accdegrvecnew{}{}(i)=\deg(G_i,\accdegrvecnew{}{},w_i)\qquad (i=1,\ldots,n)
  \label{eq:rec0}
\end{equation}
for a \emph{degree function} $\deg:\{-1,0,1\}^{n}\times\D^n\times\D\to\D$
satisfying the property that the order of arguments does not matter, i.e.\ for any permutation matrix $P$, we have
\begin{equation}
  \deg(g,d,w)=\deg(gP^{-1},Pd,w)
  \label{eq:deg}
\end{equation}
\begin{theorem}\label{thm:equivalence}
The principle ``if two arguments start out with the same initial
weight and they share the same degree of attack and support, they have
the same acceptability degree'' follows from the existence of a
function $\deg:\{-1,0,1\}^{n}\times\D^n\times\D\to\D$ satisfying
equations (\ref{eq:rec0}) and (\ref{eq:deg}).
\end{theorem}
\begin{proof}{}
Given arguments $a_i$ and $a_j$ with $w_i=w_j$ and bijections on their
supporters and attackers, these bijections can be combined into a
permutation matrix $P$ such that
$P\accdegrvecnew{}{}=\accdegrvecnew{}{}$ and $G_iP^{-1}=G_j$.  Then,
using the equations, we have
$\accdegrvecnew{}{}(j)=\deg(G_j,\accdegrvecnew{}{},w_j)=\deg(G_iP^{-1},P\accdegrvecnew{}{},w_i)=
\deg(G_i,\accdegrvecnew{}{},w_i)=\accdegrvecnew{}{}(i)$.
\end{proof}

\begin{example}\label{ex:matrix-exp-sem}
 The matrix exponential semantics in Example~\ref{def:matrix-exp-sem})
  fails to satisfy equivalence. Consider the graph
   $$
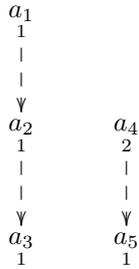
\xymatrix{
 \underset{1}{a_1} \ar@{-->}[d]  &
   \\
 \underset{1}{a_2} \ar@{-->}[d]  & 
 \underset{2}{a_4} \ar@{-->}[d]   \\
\underset{1}{a_3}  &
 \underset{1}{a_5} 
 }$$	
The  matrix exponential semantics leads to the degree $\matr{1\\2\\2.5\\2\\3}$.
  Hence the parents of $a_3$ and $a_5$ have the same degree, and
  $a_3$ and $a_5$ have the same initial weight. Still, their degrees
  are different (2.5 and 3). This means that degrees cannot be
  explained locally, but only by looking at all paths into an argument.
\end{example}
Below, we introduce \emph{direct aggregation semantics}, which has
an intuition similar to matrix exponential semantics while enjoying
the \emph{equivalence} property.

\medskip
\charac{Modularity} strengthens\footnote
{  A example semantics that satisfies \emph{equivalence}, but 
	 violates \emph{modularity} 
	 would be a semantics that  computes two aggregations,
	   one for attackers and one for supporters, and then combines these with
	   the initial weight using  a ternary influence function.
While \emph{modularity} is stronger than \emph{equivalence}, 
 all semantics (in the sense
of Def.~\ref{def:semantics}) in the literature satisfy \emph{modularity}.} 
the locality principle expressed by
\emph{equivalence} in that the degree function is a
composition of two functions:
\begin{equation}
\deg(g,d,w)=\inflp(\aggrp(g,d),w),
   \label{eq:modularity}
\end{equation}
namely an \emph{aggregation}
function $\aggrp:\{-1,0,1\}^{n}\times\D^n\to\R$ and an
\emph{influence} function $\inflp:\R\times\D\to\D$. The
compositionality expressed by this can be detailed as follows (see also Fig.~\ref{fig:modular}): First,
$\aggrp$ aggregates all the degrees of the parent arguments that
influence a given argument according to the graph structure of
$G$. Second, $\inflp$ determines how the aggregated parent arguments
actually modify the initial weight of the given
argument. The condition corresponding to equation~(\ref{eq:deg})
is stated as \emph{anonymity}-2
below. Equation~(\ref{eq:rec0}) now becomes
\begin{equation}
  \accdegrvecnew{}{}(i)=\inflp(\aggrp(G_i,\accdegrvecnew{}{}),w_i)
  \qquad (i=1,\ldots,n)
  \label{eq:rec2}
\end{equation}



By entrywise action, we extend $\aggrp$ to
$\aggr:\{-1,0,1\}^{n\times n}\times\D^n\to\R^n$ and  $\inflp$ to
$\infl:\R^n\times\D^n\to\D^n$.
Equation~(\ref{eq:rec2}) then becomes
\begin{equation}
  \accdegrvecnew{}{}=\infl(\aggr(G,\accdegrvecnew{}{}),w)
  \label{eq:rec1}
\end{equation}
%

\medskip
\charac{Anonymity} expresses that acceptability degrees are invariant
under renaming and reshuffling of the arguments. Firstly, this means that only the incidence matrix $G$
of a graph matters, while the set of arguments $\argset$ does not
 (\emph{Anonymity}-1). We thus
formulate all characteristics using incidence matrices, as already
explained above. 
Hence, as already
done in equations~(\ref{eq:rec0}), (\ref{eq:rec2}) and (\ref{eq:rec1}), we
 identify a \wasa\ with a pair $\langle {G, w}\rangle$,
omitting $\argset$.
Secondly,
anonymity means invariance under graph isomorphism. This is not a structural,
but an essential characteristics and is therefore discussed below.

With these structural prerequisites, we are now ready for the
following definition that is fundamental to our approach:
\begin{definition}[Modular Acceptability Semantics]\label{def:modularSemantics}
A modular acceptability semantics $(\D,\aggrp,\inflp)$
consists of a connected domain $\D\subseteq\mathbb{R}$ with neutral value $0\in\D$, an
aggregation function $\aggrp:\{-1,0,1\}^{n}\times\D^n\to\R$ and an
influence function $\inflp:\R\times\D\to\D$, satisfying the structural
characteristics in Table~\ref{tab:our-char}.
\end{definition}
 It is clear that, using equation~(\ref{eq:rec1}), each modular acceptability
semantics induces an acceptability semantics (see Def.~\ref{def:semantics}), but not 
every acceptability semantics is modular. 
In the following we will consider only {modular} acceptability semantics.  Assuming modularity in the formulation of other characteristics has the benefit that it simplifies 
their formulation and makes it easier to test, for a given acceptability semantics, whether 
it possesses these characteristics. Further, it enables the recombination of different influence and aggregation functions. As we will show, 
all acceptability semantics 
 (in the
sense of Def.~\ref{def:semantics}) that have been studied in the literature are modular acceptability semantics. This illustrates that {modular} acceptability semantics form a rich and interesting subclass of acceptability semantics.


\subsection{Essential Characteristics}

\charac{Anonymity-2} expresses that acceptability degrees are invariant
under bijectively reordering (via a graph isomorphism) the arguments.
It can be
expressed in the simplest form as
$$\aggrp(gP^{-1},Pd)=\aggrp(g,d)\qquad\text{(\emph{Anonymity}-2)}$$
for any permutation matrix $P$. (Graph isomorphisms can, at the level of
matrices, be expressed as permutation matrices.)

\begin{lemma}\label{lem:anonymity-simplified}
  \emph{Anonymity}-2 implies that
  if $P$ is a permutation matrix, then
$\aggr(GP^{-1},Pd)=\aggr(G,d)$ and
$\aggr(PGP^{-1},Pd)=P\aggr(G,d)$.
\end{lemma}
\begin{proof}{}
The first equation follows because $\aggr$ is computed row-wise.
Multiplying left with $P$ gives $P\aggr(GP^{-1},Pd)=P\aggr(G,d)$. Since
multiplying $P$ from the left swaps rows, we have
$\aggr(PGP^{-1},Pd)=P\aggr(GP^{-1},Pd)$. Combining these two equations
gives the second result.
\end{proof}

\charac{Independence} states that acceptability degrees for disjoint
unions of graphs are built component-wise. In general, this means
that
$$G=\matr{G_1&0\\0&G_2}\wedge w=\matr{w_1\\w_2}\to\accdegrvecnew{}{}=\matr{\accdegrvecnew{}{\langle {G_1,
       w_1}\rangle}\\\accdegrvecnew{}{\langle {G_2, w_2}\rangle}}$$
By equation~(\ref{eq:rec1}), in our modular setting, it suffices to express this using $\aggr$:
$$\aggr\Big(\matr{G_1&0\\0&G_2},\matr{w_1\\w_2}\Big)=\matr{\aggr(G_1,w_1)\\\aggr(G_2,w_2)}$$

In terms of $\aggrp$, this can be expressed as the invariance under addition
of new single disconnected components to the graph:
$$\aggrp(\matr{0&g},\matr{x\\d}) = \aggrp(g,d) = \aggrp(\matr{g&0},\matr{d\\x})$$

\medskip
\charac{Reinforcement} requires that if an attacker of an argument is weakened or 
a supporter is strengthened, then the acceptability degree of the 
argument increases.
This characteristic leads to two axioms, one for $\aggrp$ and
one for $\inflp$.

Consider a vector $g=G_j$ that is a row of the
matrix $G$, that is, we consider the $j$-th argument $a_j$.
$g$ induces a partial ordering on vectors $d$ expressing the degree of
all arguments (including the parents of our given argument $a_j$):
$$d^1\leq_g d^2\textrm{ iff for all }i=1\ldots n,\ g_id^1_i\leq g_id^2_i$$
Since $g_i$ expresses whether argument $a_i$ attacks (-1) or supports
(1) argument
$a_j$ or does neither of these (0), 
 $d^1\leq_g d^2$
can be rewritten in more conventional terms as
$$d^1\leq_{G_j} d^2\text{ iff for all } i=1,\ldots,n
\begin{cases}
d^1_i\leq d^2_i\text{ if }a_i\text{ supports } a_j \text{ and}\\
d^1_i\geq d^2_i\text{ if }a_i\text{ attacks } a_j 
\end{cases}$$
Altogether, if $g=G_j$ is the row for argument $a_j$, $d^1\leq_g d^2$
expresses that support of $a_j$ by its parents is weaker (and attack
stronger) in $d^1$ than in $d^2$.

\medskip
\emph{Reinforcement}-$\aggrp$ can now be expressed by
$$d^1\leq_g d^2\to\aggrp(g,d^1)\leq\aggrp(g,d^2)$$
Moreover, \emph{Reinforcement}-$\inflp$ expresses that a stronger
support by its parents leads to a stronger acceptability degree
of an argument, in case its initial weight $w$ remains the same:
$$s_1< s_2 \to \inflp(s_1,w)\prec \inflp(s_2,w)$$
where $v\prec w$ means that $v<w$
or $v=w=\mind$ or $v=w=\maxd$. The latter is needed since one
cannot expect strict monotonicity at the boundary points.
  
  \medskip
\charac{Initial monotonicity} means that a stronger initial weight should
also lead to a stronger  acceptability degree of an argument:
$$w_1<w_2\to \inflp(s,w_1)\prec \inflp(s,w_2).$$

  \medskip
\charac{Stability} requires that in the absence of supports or attacks the
acceptability degree of an argument coincides with its initial
weight.  Hence, $\aggrp(0,d)=0$ and $\inflp(0,w)=w$.
Here, the first argument of $\aggrp$ is the zero vector $0$, expressing
absence of supports and attacks. Moreover, the result of $\aggrp(0,d)$
(that is fed into the first argument of $\inflp$) is the real number
$0$, representing zero influence from parents.

\medskip
\charac{Continuity} excludes chaotic behaviour, where small differences
in the initial weight lead to widely divergent acceptability degrees
(cf.\ also \cite{DBLP:conf/kr/RagoTAB16}). This amounts to requiring
$\inflp$ to be continuous (we call this
\emph{continuity-$\inflp$}), as well as $\aggrp$ in the second argument
(called \emph{continuity-$\aggrp$}).

\medskip
\charac{Neutrality} expresses that, given an argument $a_k$ with neutral
acceptability degree $0$, one can remove all attack and
support relationships that $a_k$ is involved in, since $a_k$ has no impact
on the acceptability degrees of rest of the arguments.
This is formalised as:
$$d\geq 0\wedge(\forall j\neq k . g_j=g'_j)\wedge d_k=0 \to \aggrp(g,d)=\aggrp(g',d).$$
Here, the conjunct $\forall j\neq k . g_j=g'_j$ of the
 antecedens expresses
that $g'$ and $g$ coincide except possibly regarding attack and/or
support relationships involving $a_k$.

\medskip
\charac{Directionality} captures the idea that attack and support are
directed relationships, that is, the attacker (supporter) influences
the attacked (supported), but not vice versa. While in
\cite{DBLP:conf/ecsqaru/AmgoudB17} this is expressed using paths in
the argumentation graph,
we here express it in terms of $\aggrp$, which considers only one step.
The formulation in \cite{DBLP:conf/ecsqaru/AmgoudB17} follows from
this using induction over the path and equation (\ref{eq:rec2}),
see Thm.~\ref{thm:French-characteristics}. 
If two arguments $a_1, a_2$ have the same number of attackers and supporters with identical acceptability degrees, then the aggregation of effect of the parents of $a_1$ equals   the aggregation of effect of the parents of $a_2$.  
Formally, given a row $g\in\{-1,0,1\}^n$ of the argumentation matrix $G$,
define an equivalence relation on degree vectors $d\in\D^n$ by
\begin{quote}
  $d^1\equiv_g d^2$ iff (for $j=1,\ldots,n$, $g_j\neq 0$ implies $d^1_j= d^2_j$)
\end{quote}
Here, $d^1$ (or $d^2$) is the vector of degrees of the possible parents of
$a_1$ (or $a_2$) respectively.
Then \emph{directionality} can be expressed as
$$d^1\equiv_g d^2 \to \aggrp(g,d^1)=\aggrp(g,d^2).$$

\medskip
\charac{Parent Monotonicity} is called monotony in
\cite{DBLP:conf/ijcai/AmgoudB16,DBLP:conf/ijcai/AmgoudBDV17}
and (when combined with \emph{counting}) bi-variate monotony in
\cite{DBLP:conf/ecsqaru/AmgoudB17}. It requires that, for any given
argument $a$ in a \wasa, if one weakens or removes attackers of $a$ or
strengthens or adds supporters of $a$, then this leads to a stronger
or equal acceptability degree of $a$. This is expressed by
two monotonicity requirements here, which together ensure
the desired monotonicity. The first one is
\emph{Parent Monotonicity}-$\aggrp$:
$$(d\geq 0\wedge g_1\leq g_2)\to \aggrp(g_1,d)\leq \aggrp(g_2,d)$$
Note that we include $d\geq 0$ in the antecedent because we want to 
allow for acceptability semantics where the support of an argument with 
acceptability less than the neutral degree $0$  does not  strengthen the supported argument.  
The second monotonicity requirement is
\emph{Parent Monotonicity}-$\inflp$:
$$s_1\leq s_2 \to \inflp(s_1,w)\leq\inflp(s_2,w),$$

The idea underlying \charac{strengthening} and \charac{weakening}
as in \cite{DBLP:conf/ecsqaru/AmgoudB17} is that
attackers and supporter can counter-balance each other, but the
stronger of them wins. More precisely,
weakening states that if attackers overcome supporters, the degree of
an argument should be less than its initial weight. Dually,
strengthening means that if supporter overcome attackers, the degree
of an argument should be greater than its initial weight.
As usual, we split our consideration into axioms for $\aggrp$ and
for $\inflp$. Concerning the latter,
\emph{strengthening}-$\inflp$ requires that the acceptability degree of a supported
argument\footnote{Understood in the sense that $\aggrp$ has a positive
value.} is higher than its initial weight. This is
expressed as
$$s>0 \to w\prec \inflp(s,w).$$
Dually, \emph{weakening}-$\inflp$ requires that the acceptability degree of an
attacked argument is lower than its initial weight, thus  
$$s<0 \to \inflp(s,w)\prec w.$$ 
Note that in \cite{DBLP:conf/ecsqaru/AmgoudB17},
$\mind$ or $\maxd$ need to be explicitly excluded.
The use of $\prec$ instead of $<$ frees us from doing so.

\medskip
We now come to the $\aggrp$ parts of strengthening and weakening.
These are the most complex of all characteristics. This complexity is
already visible in \cite{DBLP:conf/ecsqaru/AmgoudB17}.  The
general idea is that if supporters outweigh attackers, the supporters
strengthen the argument.  \emph{Strengthening}-$\aggrp$ can be
axiomatised as follows:
\begin{quote}
Given a vector $g\in\{-1,0,1\}^n$
(that typically is a row of the attack/support matrix $G$)
and a vector $0\leq d\in\D^n$ of parent's degrees, 
  if $f:\{1,\ldots,n\}\to\{1,\ldots,n\}$ is a bijection\footnote{In \cite{DBLP:conf/ecsqaru/AmgoudB17}, a partial injective function (defined
    only on attacking arguments) is used.
    However note that any partial injection on $\{1,\ldots,n\}$
    can be extended to a total bijection on $\{1,\ldots,n\}$.
    The additional function values of the bijection can be arbitrarily
    chosen as they do not matter here.}
  mapping attacks
  to supports that are at least as strong,
  i.e.\ for $i=1,\ldots,n$,
  $$ g_i=-1\to\Big( g_{f(i)}=1\wedge d_i\leq d_{f(i)}\Big),$$
  then we have $\aggrp(g,d)\geq 0$\footnote{Note that this non-strict
    part does not occur in \cite{DBLP:conf/ecsqaru/AmgoudB17}, but nevertheless
    we consider it to be useful}.  
  Moreover, if for some attack, the corresponding support is strictly
  stronger, i.e.\ for some $i$ with $ g_i=-1$, $d_i< d_{f(i)}$,
  or if there is some non-zero support that does not correspond to an attack,
  i.e.\ there is some $i$ with $ g_i\neq -1$,
  $ g_{f(i)}=1$ and $d_{f(i)}\neq 0$,
  we have $\aggrp(g,d)> 0$.
\end{quote}

Dually, \emph{weakening}-$\aggrp$ can be axiomatised as follows:
\begin{quote}
Given $g\in\{-1,0,1\}^n$ and  $0\leq d\in\D^n$,
  if $f:\{1,\ldots,n\}\to\{1,\ldots,n\}$ is a bijection
  that satisfies, for $i=1,\ldots,n$,
  $$ g_i=1\to\Big( g_{f(i)}=-1\wedge d_i\leq d_{f(i)}\Big),$$
  then we have $\aggrp(g,d)\leq 0$.
   Moreover, if for some support, the corresponding attack is strictly
  stronger, i.e.\ for some $i$ with $ g_i=1$, $d_i< d_{f(i)}$,
  or if there is some non-zero attack that does not correspond to a support,
  i.e.\ there is some $i$ with $ g_i\neq 1$,
  $ g_{f(i)}=-1$  and $d_{f(i)}\neq 0$,
  we have $\aggrp(g,d)<0$.
 
\end{quote}
We now develop an equivalent formulation at the vector level,
i.e.\ without using indices.
We need two auxiliary
functions extracting the non-zero supporting or attacking parents
of a vector $g\in\{-1,0,1\}^n$:
$$\supps(g,d)_{i} = \begin{cases}1& g_i=1, d_i\neq0\\0&otherwise\end{cases}$$
  $$\atts(g,d)_{i} = \begin{cases}1& g_i=-1, d_i\neq0\\0&otherwise\end{cases}$$
    Filtering out attacks and supports with neutral degree $0$ is needed
    for the case where the ``surplus'' support or attack
needs to be non-zero (see above).
The role of the bijection $f$ is now played by
a permutation matrix $P$.
The vectorised form of \emph{strengthening}-$\aggrp$  is then:
for a permutation matrix $P$ and $0\leq d\in\D^n$, we have that
$$\begin{array}{c}
  \atts(g,d)\leq \supps(Pg,Pd)\\
  \textrm{and}\\
  \atts(g,d)\circ d\leq \supps(Pg,Pd)\circ Pd\\
  \textrm{imply}\\
  \aggrp(g,d)\geq 0,
  \end{array}$$
and if one of the two first inequalities is strict, the third one is
so as well. Note that $\circ$ denotes entrywise multiplication
of vectors (Hadamard product). 

Dually, the vectorised form of \emph{weakening}-$\aggrp$ is:
for a permutation matrix $P$ and $0\leq d\in\D^n$, we have that
$$\begin{array}{c}
  \supps(g,d)\leq \atts(Pg,Pd)\\
  \textrm{and}\\
  \supps(g,d)\circ d\leq \atts(Pg,Pd)\circ Pd\\
  \textrm{imply}\\
  \aggrp(g,d)\leq 0,
  \end{array}$$
and if one of the two first inequalities is strict, the third one is
so as well.


\medskip
\charac{Soundness} expresses that any difference between an initial weight and 
the acceptability degree of an argument is  caused by some supporting (attacking, respectively) argument.  This leads to
$$\inflp(s,w)\neq w\to s\neq 0.$$

\subsection{Optional Characteristics}
 \charac{Compactness} expresses that $\D$ has a minimum $\mind$ and  a maximum  $\maxd$.  Boundedness as 
defined in   \cite{DBLP:conf/ijcai/AmgoudB16}  follows directly from our notions of compactness and 
parent monotony. 

\medskip
\charac{Resilience} ensures that an argument may have a perfect
($\maxd$) or a worthless ($\mind$) degree only if its initial weight
is of the same value:
if $\maxd$ is defined, then
$$w<\maxd \to \inflp(s,w)<\maxd\text{ (\emph{resilience-max})}$$
and if $\mind$ is defined, then
$$w>\mind \to \inflp(s,w)>\mind\text{ (\emph{resilience-min}).}$$
If neither $\maxd$ nor $\mind$ are defined, then \emph{resilience}
is satisfied vacuously.

\medskip
\charac{Stickiness} expresses, dually, that arguments
cannot escape the minimum or maximum value.
\emph{Stickiness-min} means that $\mind$ 
is defined, and
$$\inflp(s,\mind)=\mind.$$
Dually, \emph{Stickiness-max} means that $\maxd$ 
is defined, and
$$\inflp(s,\maxd)=\maxd.$$

\medskip 
\charac{Franklin} expresses that if argument $a$ is attacked by $b$ and
supported by $c$ and the acceptability degrees of $b$ and $c$ are
identical, then $c$ and $b$ neutralise each other (w.r.t.\ $a$).
Formally, this can be expressed as
$$g_i=-g_j\wedge d_i=d_j\to\aggrp(g,d)=\aggrp(g[i:0][j:0],d)$$
where $$(g[i:x])_k=\begin{cases}x & i=k\\g_k&otherwise\end{cases}$$ 
Here, $g$ represents the vector of $a$'s parents, while $g[i:0][j:0]$
is like $g$, but with one attacker and one supporter removed.

\medskip
\charac{Counting} requires that any additional support (attack)
increases (decreases) the acceptability more. Formally,
$$(d\geq0\wedge \forall i{\neq} k.g_i{=}h_i) \to
\sgn(\aggrp(g,d)-\aggrp(h,d)) =\sgn(g_k-h_k).$$
The condition $\forall i{\neq} k.g_i{=}h_i$ ensures that the
parent vectors $g$ and $h$ can differ only at position $k$.
Now for example if $g_k=1$ (support) but $h_k=-1$ (attack),
then $\aggrp(g,d)>\aggrp(h,d)$ (the support increases the aggregated
value). 

Note that the counting axiom typically is useful if $g_k \neq h_k$.
However, we did not add $g_k \neq h_k$ as an assumption, because in
case that $g_k = h_k$, then the first assumption implies $g=h$, and
the axiom then just states the trivial logical consequence
$\aggrp(g,d)=\aggrp(h,d)$.

\medskip
We have required $d\geq0$ in various axioms, because
$\aggrp$ might ignore parents with degree less than $0$.
On the other hand, it is quite natural not to ignore such parents.
\charac{Symmetry} expresses that a support with degree $d$
is equivalent to an attack with degree $-d$ (and vice versa). Formally,
$$\aggrp(g,d)=\aggrp(-g,-d).$$

\bigskip

Altogether, we arrive at the principles summarised in
Table~\ref{tab:our-char}. Note that the axioms that are derived from the characteristics axiomatise
either $\aggrp$ or $\inflp$, except for \emph{modularity}, which
provides their link. Therefore, the conditions respect the
orthogonal structure of our semantics.

\begin{definition}[Well-Behaved Modular Acceptability Semantics]\label{def:well-behaved-modularSemantics}
  A well-behaved modular acceptability semantics $(\D,\aggrp,\inflp)$
  is a modular acceptability semantics $(\D,\aggrp,\inflp)$
  satisfying the axioms for the essential
characteristics in Table~\ref{tab:our-char}.
\end{definition}


\begin{table*}
\begin{tableenv}
$$\def\arraystretch{1.3}
  \begin{array}{|l|l|}\hline
    \multicolumn{2}{|c|}{\textbf{Structural Characteristics}}\\\hline
  \textit{Anonymity-1}&    \accdegrvec{}{\langle {\argset, G, w}\rangle}=\accdegrvecnew{}{} \\\hline
  \textit{Equivalence}&   \accdegrvecnew{}{}(i)=\deg(G_i,\accdegrvecnew{}{},w_i)\qquad (i=1,\ldots,n)\\\hline
  \textit{Modularity}&    \accdegrvecnew{}{}(i)=\inflp(\aggrp(G_i,\accdegrvecnew{}{}),w_i)\qquad (i=1,\ldots,n) \\\hline
     \multicolumn{2}{|c|}{\textrm{Conditions on $\aggrp$}}\\\hline
    \multicolumn{2}{|c|}{\textbf{Essential Characteristics}}\\\hline
    \multicolumn{2}{|c|}{\textrm{Conditions on $\inflp$}}\\\hline
 \textit{Reinforcement-$\inflp$}&  s_1< s_2 \to \inflp(s_1,w)\prec \inflp(s_2,w) \\\hline
 \textit{Initial monotonicity}&  w_1<w_2\to \inflp(s,w_1)\prec \inflp(s,w_2) \\\hline
 \textit{Stability-$\inflp$}&  \inflp(0,w)=w \\\hline
\textit{Continuity-$\inflp$}&  \inflp\textrm{ is continuous}\\\hline
    \multicolumn{2}{|c|}{\textrm{Conditions on $\aggrp$}}\\\hline
  \textit{Anonymity-2}&  \aggrp(gP^{-1},Pd)=\aggrp(g,d)\textrm{ if $P$ is a permutation matrix}\\\hline
 \textit{Independence} & \aggrp(\matr{0&g},\matr{x\\d}) = \aggrp(g,d) = \aggrp(\matr{g&0},\matr{d\\x}) \\\hline
 \textit{Reinforcement-$\aggrp$}& d^1\leq_g d^2\to\aggrp(g,d^1)\leq\aggrp(g,d^2) \\\hline
 \textit{Parent monotonicity-$\aggrp$}&  d\geq 0\wedge g_1\leq g_2\to \aggrp(g_1,d)\leq \aggrp(g_2,d) \\\hline
 \textit{Stability-$\aggrp$}&  \aggrp(0,d)=0 \\\hline
\textit{Continuity-$\aggrp$}&  \aggrp(g,\_)\textrm{ is continuous}\\\hline
 \textit{Neutrality}&d\geq 0\wedge  (\forall j\neq k . g_j=g'_j)\wedge d_k=0 \to \aggrp(g,d)=\aggrp(g',d)\\\hline
\textit{Strengthening-$\aggrp$}&
\textrm{for }d\geq0, P \textrm{ a permutation matrix},   \atts(g,d)\leq \supps(Pg,Pd)\\
&\textrm{ and }\atts(g,d)\circ d\leq \supps(Pg,Pd)\circ Pd \textrm{ imply } \aggrp(g,d)\geq 0,\\
&\textrm{If one of the two first inequalities is strict, the third one is
so as well}  \\\hline
\textit{Weakening-$\aggrp$}&
\textrm{for }d\geq0, P \textrm{ a permutation matrix},   \supps(g,d)\leq \atts(Pg,Pd)\\
&\textrm{ and }\supps(g,d)\circ d\leq \atts(Pg,Pd)\circ Pd \textrm{ imply } \aggrp(g,d)\leq 0,\\
&\textrm{If one of the two first inequalities is strict, the third one is
so as well}  \\\hline
    \multicolumn{2}{|c|}{\textbf{Entailed Characteristics}}\\\hline
    \multicolumn{2}{|c|}{\textrm{Conditions on $\inflp$}}\\\hline
 \textit{Parent monotonicity-$\inflp$}&  s_1\leq s_2\to \inflp(s_1,w)\leq \inflp(s_2,w) \\\hline
 \textit{Soundness}&  \inflp(s,w)\neq w\to s\neq 0 \\\hline
 \textit{Strengthening-}\inflp&  s>0 \to w\prec \inflp(s,w) \\\hline
 \textit{Weakening-}\inflp&  s<0 \to \inflp(s,w)\prec w\\\hline
    \multicolumn{2}{|c|}{\textrm{Conditions on $\aggrp$}}\\\hline
\textit{Directionality}& d^1\equiv_g d^2 \to \aggrp(g,d^1)=\aggrp(g,d^2)\\\hline
    \multicolumn{2}{|c|}{\textbf{Optional Characteristics}}\\\hline
    \multicolumn{2}{|c|}{\textrm{Conditions on $\inflp$}}\\\hline
	\textit{Compactness}&  \textrm{there is  a $\mind \in \D$ and a  $\maxd \in \D$ }\\\hline		
\textit{Resilience}&  w>\mind \to \inflp(s,w)>\mind \qquad w<\maxd \to \inflp(s,w)<\maxd\\\hline
\textit{Stickiness-min}&\inflp(s,\mind)=\mind\\\hline
\textit{Stickiness-max}&\inflp(s,\maxd)=\maxd\\\hline
    \multicolumn{2}{|c|}{\textrm{Conditions on $\aggrp$}}\\\hline
 \textit{Franklin}& g_i=-g_j \wedge d_i=d_j\to\aggrp(g,d)=\aggrp(g[i:0][j:0],d)\\\hline
 \textit{Counting}&(d\geq0\wedge \forall i{\neq} k.g_i{=}h_i) \to
\sgn(\aggrp(g,d)-\aggrp(h,d)) =\sgn(g_k-h_k)\\\hline
 \textit{Symmetry}&\aggrp(g,d)=\aggrp(-g,-d)\\\hline
  \end{array}$$
\vspace{-2ex}
\caption{\label{tab:our-char}Overview of characteristics. $\prec$, $\leq_g$, $\equiv_g$,
  $\supps$ and $\atts$ are
defined in the text.}
\end{tableenv}
\end{table*}

\begin{theorem}\label{thm:char-independence}
All essential axioms are independent of each other.
\end{theorem}

\begin{theorem}\label{thm:char-entailment}
\emph{Reinforcement-$\inflp$}  entails  \emph{parent monotonicity-$\inflp$}.
\emph{Stability} entails \emph{soundness}.
\emph{Reinforcement-$\inflp$} and \emph{stability} together entail  \emph{strengthening-$\inflp$}
and \emph{weakening-$\inflp$}.
\emph{Reinforcement-$\aggrp$}  entails  \emph{directionality}.
\end{theorem}

We can show that our characteristics entail all characteristics from \cite{DBLP:conf/ecsqaru/AmgoudB17} listed in Table~\ref{tab:characteristics}.
\begin{theorem}\label{thm:French-characteristics}
  Any well-behaved modular acceptability semantics \bs (i.e., a semantics that satisfies our structural and essential characteristics
  in Table~\ref{tab:characteristics} above) satisfies the following
  characteristics
  defined in \cite{DBLP:conf/ecsqaru/AmgoudB17} (with
$0$ replaced by $\mind$ and 1 by $\maxd$):
anonymity,
bi-variate independence,
bi-variate equivalence,
bi-variate directionality (under the assumption of
convergence as in Def.~\ref{def:conv}),
stability,
neutrality,
the parent monotonicity part of bi-variate monotony
non-strict bi-variate reinforcement,
weakening, strengthening
and resilience.
\end{theorem}
Note that \emph{strict bi-variate reinforcement} from
\cite{DBLP:conf/ecsqaru/AmgoudB17} is not necessarily entailed,
because it would assume that, for example, if a supporter is strictly
strengthened, the degree of an argument strictly increases.  Assuming
this as an axiom would rule out e.g.\ top-based semantics as discussed
in the next section.  Hence, our version of \emph{reinforcement} is
more liberal and does not impose this strictness requirement.

\section{Implementing the modular parts of a semantics} \label{sec:semantics-ag}	

Several different principles to implement $\aggrp$ have been studied in
the literature. We formulate them here concisely and add two
sigmoid variants.

\begin{description}
  \item[\emph{Sum}] all supporting and attacking
    arguments are considered and summed up, while support and attack
    cancel each other out \cite{DBLP:conf/ijcai/AmgoudB16,DBLP:conf/ijcai/AmgoudBDV17,DBLP:conf/ecsqaru/AmgoudB17,DBLP:journals/corr/MossakowskiN16}. This is realised by using matrix
    multiplication for argument aggregation
    $$\aggrp^{\mathit{sum}}(g,d)=gd$$
    This implies $\aggr^{\mathit{sum}}(G,d)=Gd$. Note that with \emph{sum},
    attacks with negative degree are effectively supports and vice versa.
\item[\emph{Sum-pos}] is a variant of \emph{sum} where
  only parents with positive degrees are taken into account. That is,
  $$\aggrp^{\mathit{sum\textrm{-}pos}}(g,d)=\sum_{i=1,\ldots,n;d_i\geq 0}g_id_i.$$
  Note that the semantics in \cite{DBLP:conf/ijcai/AmgoudB16,DBLP:conf/ijcai/AmgoudBDV17,DBLP:conf/ecsqaru/AmgoudB17} do not use domains with negative
  values, hence, one could equally well argue that their semantics
  are based on \emph{sum} or \emph{sum-pos}.
  \item[\emph{Sum}-$\sigma$] This is similar to \emph{sum}, but the
  second argument is first fed into the inverse of a sigmoid function:
  $\aggrp^{\mathit{sum}\textrm{-}\sigma}(g,d)=g\sigma^{-1}(d)$.
  A sigmoid function is a bijection
$\sigma:\R\to(-1,1)$ that is continuous and strictly
increasing.  For definiteness, we will use the
 hyperbolic tangent $$\sigma(x)=\tanh(x).$$
  \item[\emph{Top}]
only the strongest supporter
and the strongest attacker have influence \cite{DBLP:conf/ijcai/AmgoudB16,DBLP:conf/ijcai/AmgoudBDV17}. Simultaneously, only parents with positive degrees are taken into account,
like for \emph{sum-pos}\footnote{Of course, it is also possible to consider
  a variant where all parents are taken into account, such that the strongest
  parents can have negative degree. However, we found it more natural to require
a positive degree for ``top'' parents.}. This
can be achieved by
$$\aggrp^{\mathit{top}}(g,d)=\mathit{top}(g,d)d$$
where $\mathit{top}(g,d)$ removes those entries from $g$ which
do not correspond to a strongest support or strongest attack:
$$\mathit{top}(g,d)_{i}=\twocase{g_{i}}{d_k<d_i\textrm{ for } 1\leq k< i, sgn(g_{k})=sgn(g_{i})\\
  &\textrm{and }d_k\leq d_i\textrm{ for }i<k\leq n, sgn(g_{k})=sgn(g_{i})\\
  &\textrm{and } d_i\geq 0}{0}$$

\ednote{do we want Top to ignore parents with negative degrees?}
Note that in case of several equally strong arguments, this first
one is chosen for definiteness.
%
  \item[\emph{Top}-$\sigma$] like \emph{top}, but again the second
  argument is first fed into $\sigma^{-1}$ (cf.\ \emph{sum}-$\sigma$).
\item[\emph{Reward}] the number of supporters is more
    important than their quality \cite{DBLP:conf/ijcai/AmgoudB16}.  Let
    the number of founded (i.e.\ non-neutral) arguments be represented by
    $n=g\cdot\abs(\sgn(d))$. $\abs$ and $\sgn$ are taken entrywise,
    and $\abs(\sgn(d))$ is a vector that has a 1 for each argument
    with nonzero degree. $g\cdot\abs(\sgn(d))$ counts these nonzero
    arguments, where attacks are counted negatively (extending
    the framework of \cite{DBLP:conf/ijcai/AmgoudB16}, which
    is restricted to supporters only).
    The quality of arguments is computed by $s=gd$, as for \emph{sum}.
    Then
    $$\aggrp^{\mathit{reward}}(g,d)=\begin{cases}0& n=0\\
    \frac{s}{|n|2^{|n|}}+\sgn(s)\sum_{j=1}^{|n|-1}\frac{1}{2^j}
    &otherwise\end{cases}$$

      \item[\emph{Card}] a second version of the principle ``the number of attackers is more
    important than their quality'' is given in 
    \cite{DBLP:conf/ijcai/AmgoudBDV17} (we here extend this
    to supporters as well).  Let
    $n$ and $s$ be as under \emph{reward}. Then
    $$\aggrp^{\mathit{card}}(g,d)=\begin{cases}0& n=0\\n+\frac{s}{|n|}&otherwise\end{cases}$$
    Note that
attacks count negatively here, and we get the formula
of \cite{DBLP:conf/ijcai/AmgoudBDV17} by using a negative fraction for
$\inflp$.
\end{description}



\begin{table*}

\begin{tableenv}
$$\def\arraystretch{1.6}
  \begin{array}{|l|l|p{0.4cm}|p{0.4cm}|p{0.4cm}|p{0.4cm}|p{0.4cm}|p{0.4cm}|p{1.9cm}|p{2.5cm}|}\hline
 \textbf{name}   &\aggrp(g,d)=&\!\!\!\begin{turn}{90}\textbf{Conti\-nuity-$\aggrp$}\end{turn}& \!\!\!\begin{turn}{90}\textbf{Neutra\-lity}\end{turn}& \!\!\!\begin{turn}{90}\textbf{Direc\-tio\-nality}\end{turn}& \!\!\!\begin{turn}{90}\textbf{Frank\-lin} \end{turn}& \!\!\!\begin{turn}{90}\textbf{Count\-ing} \end{turn}& \!\!\!\begin{turn}{90}\textbf{Symme\-try} \end{turn}& \textbf{requirements} & \textbf{ref.}\\\hline
   \multicolumn{10}{|c|}{\emph{Suited for unipolar graphs only}}     	\\\hline	
 \mathit{reward}&   \frac{s}{|n|2^{|n|}}+\sgn(n)\sum_{j=1}^{|n|-1}\frac{1}{2^j} &\multicolumn{1}{c|}{\checkmark}& \multicolumn{1}{c|}{\checkmark}& \multicolumn{1}{c|}{\checkmark}& \multicolumn{1}{c|}{\checkmark} & \multicolumn{1}{c|}{-}& \multicolumn{1}{c|}{-}&  \textrm{supports only}  $\D = [0,1]$ 
 & \textrm{\cite{DBLP:conf/ijcai/AmgoudB16}}\\\hline
 \mathit{card}&  n+\frac{s}{|n|}
 &\multicolumn{1}{c|}{\checkmark}& \multicolumn{1}{c|}{\checkmark}& \multicolumn{1}{c|}{\checkmark}& \multicolumn{1}{c|}{\checkmark} & \multicolumn{1}{c|}{-} &\multicolumn{1}{c|}{-} &  \textrm{attacks only} $\D = [0,1]$ 
 & \textrm{\cite{DBLP:conf/ijcai/AmgoudBDV17}}\\[1ex]\hline
 & n=g\cdot\abs(\sgn(d)), s=gd &&&&&&&& \\\hline
   \multicolumn{10}{|c|}{\emph{Suited for bipolar graphs}}     	\\\hline
 \textit{sum}&  gd & \multicolumn{1}{c|}{\checkmark}&\multicolumn{1}{c|}{\checkmark}& \multicolumn{1}{c|}{\checkmark}& \multicolumn{1}{c|}{\checkmark}& \multicolumn{1}{c|}{\checkmark} & \multicolumn{1}{c|}{\checkmark}& \multicolumn{1}{c|}{-}  & \textrm{\cite{DBLP:conf/ijcai/AmgoudB16}}\\\hline
 \textit{sum}\textrm{-}\textit{pos}& \sum_{i=1,\ldots,n;d_i\geq 0}g_id_i & \multicolumn{1}{c|}{\checkmark}&\multicolumn{1}{c|}{\checkmark}& \multicolumn{1}{c|}{\checkmark}& \multicolumn{1}{c|}{\checkmark}& \multicolumn{1}{c|}{\checkmark} & \multicolumn{1}{c|}{-}& \multicolumn{1}{c|}{-}  & \textrm{\cite{DBLP:conf/ijcai/AmgoudB16}}\\\hline
 \textit{sum}\textrm{-}\sigma&  g\sigma^{-1}(d) & \multicolumn{1}{c|}{\checkmark}&\multicolumn{1}{c|}{\checkmark}& \multicolumn{1}{c|}{\checkmark}& \multicolumn{1}{c|}{\checkmark}& \multicolumn{1}{c|}{\checkmark} & \multicolumn{1}{c|}{\checkmark}& \multicolumn{1}{c|}{-}  & \textrm{\cite{DBLP:journals/corr/MossakowskiN16}}\\\hline
 \mathit{top}&  \mathit{top}(g,d)d & \multicolumn{1}{c|}{\checkmark}&\multicolumn{1}{c|}{\checkmark}& \multicolumn{1}{c|}{\checkmark}& \multicolumn{1}{c|}{-} & \multicolumn{1}{c|}{-}& \multicolumn{1}{c|}{-}& \multicolumn{1}{c|}{-} & \textrm{\cite{DBLP:conf/ijcai/AmgoudB16}}\\\hline
 \mathit{top}\textrm{-}\sigma&  \mathit{top}(g,\sigma^{-1}(d))\sigma^{-1}(d) & \multicolumn{1}{c|}{\checkmark}&\multicolumn{1}{c|}{\checkmark}& \multicolumn{1}{c|}{\checkmark}& \multicolumn{1}{c|}{-} & \multicolumn{1}{c|}{-}&\multicolumn{1}{c|}{-} &\multicolumn{1}{c|}{-}  & \\\hline
  \end{array}
$$

\bigskip

\begin{tabular}{|c|c|c|l|c|p{0.3cm}|p{0.3cm}|p{0.3cm}|p{3cm}|}\hline
  \textbf{name} &$\D $&   $\inflp(s,w)=$& $s$ & \begin{turn}{90}\textbf{Com\-pact\-ness}\end{turn} &\begin{turn}{90}\textbf{Resili\-ence}\end{turn}& \begin{turn}{90}\textbf{Sticki\-ness-min}\end{turn} & \begin{turn}{90}\textbf{Sticki\-ness-max}\end{turn} &\textbf{ref.} \\\hline
   \multicolumn{9}{|c|}{\emph{Suited for unipolar graphs only}}     	\\\hline	
  multilinear &$[0,1]$&$w+(1-w)s$&$[0,1]$& \checkmark & -- & --& \checkmark & {\cite{DBLP:conf/ijcai/AmgoudB16}}\\\hline
  positive fractional &$[0,1]$&$\largefrac{w+s}{1+s}$&$s{\geq} 0$& \checkmark & 
   \checkmark& -- & \checkmark&
    {\cite{DBLP:conf/ijcai/AmgoudB16}}\\\hline
  negative fractional &$[0,1]$&$\largefrac{w}{1-s}$&$s{\leq} 0$ & \checkmark& \checkmark& -- & -- &
   {\cite{DBLP:conf/ijcai/AmgoudBDV17}}\\\hline
      \multicolumn{9}{|c|}{\emph{Suited for bipolar graphs}}     	\\\hline	
  combined fractional&$[0,1]$&$\begin{cases}
\largefrac{w}{1-s}, &  s \leq 0\\
\largefrac{w+s}{1+s},   & s \geq 0\\
\end{cases}
$&any&\checkmark  & \checkmark& -- & -- & 
 {\cite{DBLP:journals/corr/MossakowskiN16}}\\\hline
  Euler-based&$[0,1)$&$1-\largefrac{1-w^2}{1+we^s}$&any& -- &\checkmark& \checkmark&  -- & {\cite{DBLP:conf/ecsqaru/AmgoudB17}}\\\hline
  linear($\delta$)&$\R$&$\largefrac{s}{\delta}+w$&any& -- & \checkmark& --&  -- &  {\cite{DBLP:journals/corr/MossakowskiN16}}\\\hline
  sigmoid($\delta$) &$(-1,1)$&$
	\sigma(\largefrac{s}{\delta}+ \sigma^{-1}(w))$  &any  & -- & \checkmark& -- &  -- &
	      {\cite{DBLP:journals/corr/MossakowskiN16}}\\\hline
QMax&$[0,1]$&$\begin{cases}
w-w\largefrac{s^2}{1+s^2},&\!\!\!\!s \leq 0\\
w+(1-w)\largefrac{s^2}{1+s^2},&\!\!\!\!s \geq 0\\
\end{cases}
$&any&\checkmark  & \checkmark& -- &  -- &
 {\cite{Potyka18}}\\\hline              
\end{tabular}
\caption{Overview of different implementations of $\aggrp$ and $\inflp$.  $\mathit{top}$ and $\sigma$ are
defined in the text. \label{tab:aggr-impl}}

\end{tableenv}

\end{table*}

%

\begin{theorem}\label{thm:aggr-char}
  The implementations of $\aggrp$
\emph{sum},
\emph{sum-pos},
\emph{sum}-$\sigma$,
\emph{top}, and
\emph{top}-$\sigma$ 
  satisfy all structural and essential characteristics as shown in
  Table.~\ref{tab:aggr-impl}.
 If
  \emph{reward} is restricted 
  to    graphs with support relations only and    $\D=[0,1]$,
and if 
\emph{card} is restricted to
  graphs with attack relations only and $\D=[0,1]$,
then   \emph{reward}
  and
  \emph{card}
    satisfy  all structural and essential characteristics as shown in
  Table.~\ref{tab:aggr-impl}.\footnote {These restrictions are
    exactly those made in the original papers introducing these
    semantics
    \cite{DBLP:conf/ijcai/AmgoudB16,DBLP:conf/ijcai/AmgoudBDV17}.}   

\end{theorem}

Indeed, without the restrictions in Theorem \ref{thm:aggr-char},
\emph{reward} satisfies neither \emph{parent monotonicity}-$\aggrp$
nor \emph{weakening}, and \emph{card} does not satisfy \emph{parent
  monotonicity}-$\aggrp$.  This shows that these aggregation functions
are only suitable for unipolar argumentation graphs. The other
implementations of $\aggrp$ may be applied to bipolar argumentation
graphs (even if also some of them have been invented for unipolar semantics).



Coming now to possible influence functions, various implementations of
$\inflp$ from the literature are listed in Table~\ref{tab:aggr-impl}.
The implementations of $\inflp$ are partially motivated by the need to ensure that a semantics is 
well-defined and that the essential characteristics are met. 
However, they also represent interesting choices about the role of the minimum and maximum acceptability degrees \mind and \maxd. The multilinear, positive fractional, negative
fractional, the combined fractional influence functions and QMax are all compact, thus, \mind and 
\maxd are acceptability degrees that may be assigned to arguments. But except in the case 
of the multilinear influence function, \mind and  \maxd play a special role: 
\mind and \maxd are not accessible to arguments that were not 
assigned these values already as initial weights (resilience). Thus, if an argument was not initially weighted as maximum acceptability, the amount of support from other arguments does not matter, it will never reach \maxd as acceptability degree. (Analogously for attacks and \mind.) 

 The Euler-based influence function displays an interesting asymmetry: it has a minimum acceptability degree but no maximum. Thus, within the context 
of the Euler-based influence function, the `perfect acceptability degree' (the supremum) may only be approximated, while it is possible to assign arguments  the lowest acceptability. And if 
an argument is initially weighted as \mind, then its acceptability degree is \mind{} -- regardless 
of the strength of its support from other arguments (\emph{stickiness-min}). 

Both the linear and the sigmoid influence function are defined with respect to a damping factor $\delta$, which dampens the role of the aggregation argument $s$. The linear influence function 
is the only one which uses $\R$ as the value space and, thus, lacks both \mind and \maxd. The sigmoid influence function is the result of mapping $\R$ onto the interval $(-1,1)$.
Hence, for the 
sigmoid influence function, both the supremum and
the infimum of all acceptability degrees may only be approximated, but never reached.

\medskip
\begin{theorem}\label{thm:in-char}
All implementations of $\inflp$ shown in
Table~\ref{tab:aggr-impl} satisfy all the structural and essential characteristics for $\inflp$
shown in Table~\ref{tab:characteristics},
if the listed conditions on $s$ are respected.
\end{theorem}

Note that functions $\inflp$ acting on positive $s$ or on negative $s$
only are suited for unipolar graphs only. This is because \emph{weakening}
and \emph{strengthening} imply that $s=\aggrp(g,d)$ takes both positive
and negative values for suitable bipolar graphs.

\begin{corollary}
	Every combination of one of the aggregation functions \emph{sum},
\emph{sum-pos},
\emph{sum}-$\sigma$,
\emph{top}, and
\emph{top}-$\sigma$ 
  with one of the influence functions 
  \emph{combined fractional}, \emph{Euler-based},   \emph{linear}($\delta$)
  \emph{sigmoid}($\delta$), \emph{QMax}
  yields a well-behaved modular acceptability semantics. 
\end{corollary}

\begin{table*}
\begin{tableenv}
	\begin{tabular}{|p{2.8cm}|c|p{1.8cm}|p{2cm}|p{2.5cm}|p{3cm}|}\hline
  \textbf{semantics} &$\aggrp$ & $\inflp$  &\textbf{converges for} &\textbf{counter\-example}&\textbf{ref.} \\\hline
  \multicolumn{6}{|c|}{\textbf{Semantics not covering mixed support/attack graphs}}\\\hline
  aggregation-based &sum&positive fraction&  supports& $G$ has attacks&\cite{DBLP:conf/ijcai/AmgoudB16}\\\hline
  weighted h-categorizer &sum&negative fraction&  attacks& $G$ has supports& \cite{DBLP:conf/ijcai/AmgoudBDV17}\\\hline
  combined aggrega\-tion/h-categorizer&sum&combined fraction & supports only or attacks only& $G$ mixes supports and attacks&Exs.~\ref{ex:non-convergence-combined}, \ref{ex:divergence}\\\hline
  top-based &top&multilinear&  supports& $G$ has attacks&\cite{DBLP:conf/ijcai/AmgoudB16}\\\hline
  weighted max-based &top&negative fraction&  attacks& $G$ has supports& \cite{DBLP:conf/ijcai/AmgoudBDV17}\\\hline
  Reward-based&reward&multilinear&  supports& $G$ has attacks&\cite{DBLP:conf/ijcai/AmgoudB16}\\\hline
  Card-based&card&negative fraction& attacks& $G$ has supports& \cite{DBLP:conf/ijcai/AmgoudBDV17}\\\hline
  \multicolumn{6}{|c|}{\textbf{Semantics covering mixed support/attack graphs}}\\\hline
  Euler-based&sum&Euler-based&$\ind(G)\leq 4$ &$\ind(G)=6$&\cite{DBLP:conf/ecsqaru/AmgoudB17}, Thm.~\ref{thm:convergence-euler}, Ex.~\ref{ex:divergence}\\\hline
  max Euler-based&top&Euler-based&all graphs&---&Thm.~\ref{thm:max-Euler-convergence}\\\hline
  direct aggregation&sum&linear($\delta$)& $\ind(G)<\delta$& $\ind(G)\geq \delta+1$&Thm.~\ref{thm:direct-aggregation-convergence}, Exs.~\ref{ex:divergence}, \ref{ex:divergence-direct-aggregation}\\\hline
  positive direct aggregation&sum-pos&linear($\delta$)&$\ind(G)<\delta$ &$\ind(G)\geq \delta+1$ & Thm.~\ref{thm:direct-aggregation-convergence}, Ex.~\ref{ex:divergence}\\\hline
  sigmoid direct aggregation&sum-$\sigma$&sigmoid($\delta$) & $\ind(G)<\delta$ & $\ind(G)\geq \delta+1$&Thm.~\ref{thm:sigmoid-direct-aggregation-convergence}, Exs.~\ref{ex:divergence}, \ref{ex:divergence-direct-aggregation}\\\hline
  damped max-based&top&linear($\delta$), with $\delta>2$& all graphs& --- &Thm.~\ref{thm:conv-max-based-aggr}\\\hline
  sigmoid damped max-based&top-$\sigma$&sigmoid($\delta$), with $\delta>2$ & all graphs& ---&Thm.~\ref{thm:sigmoid-conv-max-based-aggr}\\\hline
quadratic energy&sum&QMax& $\ind(G)\leq 1$ & $\ind(G)=2$& Thm.~\ref{thm:convergence-quadratic}, Exs.~\ref{ex:divergence} \\\hline              
\end{tabular}
\smallskip
\caption{Overview of semantics. All semantics converge for acyclic graphs (Thm.~\ref{thm:limit-acyclic}).\label{tab:sem}}
\end{tableenv}	
\vspace{-1.5em}
\end{table*}

\section{Comparison of Modular Acceptability Semantics}\label{sec:comparison}

Due to the modular structure of our approach, any implementation of
$\aggrp$ can be combined with any implementation of $\inflp$ whose domain
(for $s$) matches the range of $\aggrp$, resulting
in an acceptability semantics. Table~\ref{tab:sem} lists some of the
possible combinations, focussing on those that have appeared in the
literature and those that provide new insights. Moreover, we list the
class of graphs for which the respective semantics is known to
convergence, and counterexamples to convergence.  We also include
networks that contain only attacks or only supports in
Table~\ref{tab:sem} by treating them as special cases of bipolar
networks. In this way, we provide a better overview over approaches
from the literature and their combination into bipolar approaches. In
particular, we obtained the biplor aggregation-based/h-categorizer
semantics as a na\"ive combination of the weighted h-categorizer 
semantics for attacks from \cite{DBLP:conf/ijcai/AmgoudBDV17} with the
aggregation-based semantics for supports from
\cite{DBLP:conf/ijcai/AmgoudB16}, and below we show that this na\"ive
combination does in general not converge.


%
%
%
 
 \medskip
As Table~\ref{tab:aggr-impl} illustrates, 
our modular approach enables a choice of different acceptability semantics. But is this choice meaningful? In other words, is our definition of well-behaved modular acceptability 
semantics so restrictive that they all behave similar?

In order to answer this question we discuss a selection of modular acceptability semantics from Table~\ref{tab:aggr-impl} that are well-defined for bipolar argumentation graphs and use them to evaluate an Example. 

The Euler-based semantics was proposed in \cite{DBLP:conf/ecsqaru/AmgoudB17}. It calculates the acceptability degree of an argument by considering the \emph{sum} of the acceptability degrees  
of its parents and combining it with its initial weight via the Euler-based influence  function. The neutral acceptability degree $0$ is also the minimum acceptability degree.
The use of \emph{sum} as aggregation function represents the intuition that 
the calculation of the acceptability degree of an argument should consider 
all of its parents weighted according to their acceptability degrees. 
Unfortunately, as will discuss in 
section \ref{sec:convergence} in more detail, the Euler-based semantics does not converge for all \wasa{}s. Further, the Euler-based semantics inherits from its influence function 
the asymmetry between the minimum acceptability degree and the lack of a maximum, since its acceptability value space is $[0,1)$. 
Another notable property of the Euler-based semantics is that it treats supports and attacks asymmetrically: many supports for an argument $a$ quickly lead to an acceptability 
degree of $a$ near $1$, while many attacks  on $a$ lead near an acceptability degree of $w(a)^2$ (and never
below that). 

The direct aggregation semantics is also based on \emph{sum}. Its differs from the  Euler-based semantics because it involves a \emph{damping parameter} $\delta$ that is used to dampen the influence that parent arguments have on their children.  
Thus, the larger the $\delta$, the more important are the initial weights and the less important are attackers and supporters. For any \wasa \arggraph{} there is a $\delta$, such that the direct aggregation semantics converges for \arggraph (see section \ref{sec:convergence}). In contrast to the Euler-based semantics, both attacks and supports are treated symmetrically. Since for the direct aggregation semantics $\D = \R$, there is no acceptability minimum or maximum. 

The main motivation for the sigmoid aggregation semantics is basically to keep the acceptability degrees in a bounded interval. This is achieved by mapping the value space of the direct aggregation semantics onto (-1,1).

In \cite{DBLP:conf/ijcai/AmgoudB16}  the authors discuss three possible design choices for an  acceptability semantics: \emph{cardinality precedence}, \emph{quality precedence}, and \emph{compensation}. 
\emph{Sum} is an implementation of compensation, since it allows, for example, a 
small number of strong supporters to compensate for a large number of weak attackers. In 
contrast, cardinality precedence would favour the larger number of weak attackers over the 
few strong supporters, while quality precedence would favour the quality of the arguments of the few supporters over the number of weak attackers. 

The 
Euler-based semantics and the direct aggregation semantics are based on  \emph{sum} and 
the sigmoid aggregation semantics is based on $\mathit{sum}\textrm{-}\sigma$. Thus, all three  implement compensation. To illustrate the benefits of our approach, we consider for each an alternative semantics that is the result of replacing \emph{sum} by \emph{top} ($\mathit{sum}\textrm{-}\sigma$ by $\mathit{top}\textrm{-}\sigma$, respectively).   
The  aggregation functions \emph{top} and $\mathit{top}\textrm{-}\sigma$ are implementations of quality preference (based on the top-based semantics in \cite{DBLP:conf/ijcai/AmgoudB16}). 
The idea behind \emph{top} and $\mathit{top}\textrm{-}\sigma$ is basically to consider only  the strongest attackers and supporters of each argument. 

The direct aggregation semantics  has an interesting property: it allows for undermining supports and strengthening attacks.%
\footnote{ The same is true for its sigmoid alternative and their top-based variants.}
E.g., if an argument $a$ is supported by an argument $b$ with the acceptability degree of $-1$, then this support has the same effect as if $b$ would attack $a$ with an acceptability degree  of $+1$. Thus, effectively, $a$'s support undermines the acceptability of $b$. Vice versa, an attack  by an unacceptable argument will strengthen the argument that is attacked. For example, imagine 
that in a public debate on minimum wage the proponent argues: 
``Minimum wage should be increased, because it would improve the living conditions of poor people''. And the
 opponent responds: ``Poor people do not deserve any help, since poverty is God's punishment for sinners.'' 
We would expect that most members of the audience would reject the opponent's argument as completely unacceptable. According to the 
direct aggregation semantics this implies that the attack of the opponent leads to an increase the acceptability of the  proponent's argument. However, alternatively one could argue that unacceptable arguments should have no impact on other arguments. 
This intuition is implemented in the positive direct aggregation semantics.

\begin{example}\label{ex:arggraph2}


 $
 \arggraph^{ex2} = \left \langle 
 \left(\begin{smallmatrix}
 a_1 \\
 a_2 \\
 a_3 \\
 a_4
 \end{smallmatrix}\right), 
 \left(\begin{smallmatrix}
 0 & 0 &  0 &   0 \\ 
 1 & 1 & 1 & { 0} \\ 
 -1 & -1 &  0 & -1 \\ 
 0 & -1 &  0 &  0 
 \end{smallmatrix}\right), 
 \left(\begin{smallmatrix}
0.8 \\
 0.7 \\
 0.001 \\
 0.7
 \end{smallmatrix}\right)
 \right \rangle 
 $	\quad	

\end{example}

\begin{wrapfigure}{R}{0.25\columnwidth}
 \centering 
 $\xymatrix{
 \underset{0.8}{a_1} \ar@{-->}[r]\ar[d]   &
 \underset{0.7}{a_2}  \ar@{-->}@(r,u) \ar[dl]<0.5ex>\ar[d]  \\
 \underset{0.001}{a_3} \ar@{-->}[ur]<0.5ex>   & 
 \underset{0.7}{a_4}  \ar[l] 
 }$	
 \caption{Graphical representation of $ \arggraph^{ex2}$  \label{fig:example2}}
\end{wrapfigure}

 Table~\ref{tab:compsem} illustrates the differences between the different semantics by applying them to $ \arggraph^{ex2}$  in 
 Example~\ref{ex:arggraph2}. Due to \emph{stability}, all well-behaved modular acceptability semantics agree that
$\accdegr{\arggraph^{ex2}}{ }{a_1} = 0.8$, since $a_1$ is neither supported nor 
attacked. But otherwise the results vary significantly. E.g., $a_3$ starts with an initial weight of close to 0. Since Euler-based and max Euler-based semantics treat 0 as  the minimum degree, the attacks of the other 
arguments on $a_3$ have little effect but to push the acceptability degree of $a_3$ even closer to 0. 
By contrast, according to the other semantics $a_3$ starts out as marginally above the neutral value and the combined attacks of the other arguments push it deep into unacceptability. The main difference between the 
semantics in the second group and the  third group in  Table~\ref{tab:compsem}   is that the the sigmoid semantics provide an upper and a lower bound for the possible acceptability degrees by limiting \
$\D$ to $(-1,1)$. 
This enables a more convenient interpretation of the acceptability degrees.  

The  difference between max-based semantics and their counterparts is that they utilise \emph{top}. Thus, these 
semantics only consider the strongest attackers and supporters. Compare, for example, the evaluation of $a_2$ by 
 Euler-based and max Euler-based semantics. $a_2$ is supported by three arguments. However, max 
 Euler-based semantics only uses the strongest; this is why the acceptability degree of $a_2$ is only 0.801 
 instead of 0.894. Note that direct aggregation semantics evaluates $a_2$ lower than its \emph{top}-based 
 counterpart, i.e.,  damped max-based semantics. As in the previous case, damped max-based semantics only considers the strongest support (namely, the support of $a_2$), 
while direct aggregation considers all three  supports for $a_2$. But because $a_3$ has a negative acceptability degree, its support undermines $a_2$. 
 This effect is stronger than the additional support from $a_1$, which explains why direct aggregation semantics evaluates the acceptability degree of $a_2$ lower than damped max-based semantics.
  
Positive direct aggregation semantics ignores attacks and supports from $a_3$, since its acceptability degree is negative. Thus, for $a_2$ only the supports from $a_1$ and $a_2$ matter, which leads to an acceptability degree of 2.2  for $a_2$. Note that this local explanation of acceptability degrees of
 arguments in terms of those of their parents is possible due to
 \emph{independence}.

$\arggraph^{ex2}$ is only one example and, thus, may only be used to highlight some differences between the different modular acceptability semantics. However, 
Table~\ref{tab:compsem} illustrates that our approach is flexible enough to 
  support a wide range of acceptability semantics that reflect different design choices and philosophical intuitions. 
  It further illustrates that these choices make a difference for the evaluation of arguments. For example, according to direct aggregation semantics, $a_4$ in $\arggraph^{ex2}$ is acceptable, 
  according to damped max-based semantics its acceptability degree is neutral,
   and according to positive direct aggregation semantics it is not acceptable.

\begin{table}
\begin{tableenv}
 \begin{tabular}{|l|c|c|c|c|c|}\hline
$\texttt{Deg}_{\arggraph^{ex2}}$ & $\delta$&  $a_1$ & $a_2$ & $a_3$& $a_4$\\\thickhline
Euler-based            & -  & 0.8 & 0.894 &  0.000 & 0.604\\\hline
max Euler-based        & -  & 0.8 & 0.801&  0.000 & 0.612 \\\thickhline		
direct aggregation     & 2 & 0.8  & 1.161 & -1.039 &  0.120 \\\hline
damped max-based       & 2 & 0.8  & 1.400  & -0.699 &  0.000  \\\hline
positive direct aggregation  & 2 & 0.8 & 2.200 &  -1.499 & -0.400 \\\thickhline
sigmoid direct aggregation  & 2 & 0.8  & 0.902 &  -0.875 &  0.126 \\\hline
sigmoid damped       max-based        & 2 & 0.8  & 0.940 &  -0.699 & 0.000 \\\hline

 \end{tabular}
\caption{Acceptability semantics applied to $\arggraph^{ex2}$ in Example~\ref{ex:arggraph2}	 \label{tab:compsem}}
\end{tableenv}	
\vspace{-1.5em}
\end{table}

\section{Convergence}\label{sec:convergence}

Let us now look at the convergence properties of some of the semantics.
We first prove general results that hold for an arbitrary well-behaved modular
acceptability semantics $(\D,\aggrp,\inflp)$.

\begin{theorem}\label{thm:limit}
  Assume 
  convergence of  the following limit:
\begin{equation}
  f^0=w,\  f^{i+1}=\inflp(\aggr(G,f^i),w),\ D=\lim_{i\to\infty}f^i
  \label{eq:limit}
\end{equation}
  Then $D$ satisfies equation~(\ref{eq:rec1}), i.e.\ $D=\accdegrvecnew{}{}$.

\end{theorem}
\begin{proof}{}
Apply \textit{Continuity-$\inflp$} and
\textit{Continuity-$\aggrp$}.
\end{proof}

\begin{definition}\label{def:conv}
We call a semantics \emph{convergent} (\emph{divergent}), if the
sequence $(f^i)$ converges (diverges).
\end{definition}

By Thm.~\ref{thm:limit}, the
limit of a convergent semantics provides a solution of the fixpoint
equation~(\ref{eq:rec1}). It is possible that also a divergent semantics
has solutions of the fixpoint equation in some cases; however, it may
be difficult to compute these solutions. Direct aggregation semantics
is a special case: computation of a fixpoint solution here just means
solving the system of linear equations $D=\frac{1}{\delta}GD+w$. However,
depending on $G$ and $w$, the system may have no or more than one
solution.  Hence, here we concentrate on convergence, which always
gives us at most one solution, and (if existing) a way to compute it.

\begin{theorem}\label{thm:limit-acyclic}
  All well-behaved  modular acceptability semantics converge for all acyclic graphs.
\end{theorem}
\begin{proof}{}
For a node $a$, let $l_a$ be the length of the longest path into $a$.
Prove by induction that $f^i_a$ is constant from $i=l_a$ onwards,
using \emph{stability} and \emph{directionality}.
\end{proof}

We now come to more specific results concerning semantics using
$\aggrp=\mathit{sum}$.
\begin{theorem}\label{thm:convergence-sum}
 Assume that we use $\mathit{sum}$ or $\mathit{sum\textrm{-}pos}$ for aggregation $\aggrp$. Fix
  $\langle {G, w}\rangle$. If $$m=\sup_{i=1,\ldots,n,
  s\in\R}\left(\frac{\partial\inflp(x,w)}{\partial
  x}\big|_{(s,w_i)}\right)$$ exists and $\ind(G)<\frac{1}{m}$ (where
  $\ind(G)$ is the maximal indegree of $G$), then
  $\lim_{i\to\infty}f^i$ converges.
\end{theorem}
\begin{proof}{}
We will make use of the maximum row sum norm for matrices, defined by
$\rownorm{G}=\max_{i=1,\ldots,n}\sum_{j=1,\ldots,n}|G_{ij}|$,
of the maximum norm for vectors, defined by
$\rownorm{w}=\max_{i=1,\ldots,n}|w_i|$, and of the inequality
$\rownorm{Gw}\leq\rownorm{G}\rownorm{w}$ \cite{Horn:1985:MA:5509}.
Note that the norm coincides with the maximal indegree,
i.e.\ $\rownorm{G}=\ind(G)$.
Let $\varepsilon=m\cdot\ind(G)$. By assumption, $\varepsilon<1$.
We have
$$\begin{array}{ll}
\rownorm{f^{i+1}-f^i} &=\rownorm{\infl(Gf^i,w)-\infl(Gf^{i-1},w)}\\
&\leq m\rownorm{Gf^i-Gf^{i-1}}\\
&= m\rownorm{G(f^i-f^{i-1})}\\
&\leq m\rownorm{G}\rownorm{f^i-f^{i-1}}\\
&=m\cdot\ind(G)\rownorm{f^i-f^{i-1}}\\
&\leq\varepsilon\rownorm{f^i-f^{i-1}}.
\end{array}$$
Hence, $(f^i)$ is a Cauchy sequence and converges.
The proof for $\mathit{sum\textrm{-}pos}$ is similar, considering
that $G$ needs to be replaced by a submatrix of $G$ that also
depends of $f^i$.
\end{proof}

\begin{theorem}\label{thm:convergence-euler}
  Given $\langle {G, w}\rangle$, Euler-based semantics converges if
  $\ind(G)<\frac{4}{1-\min_i w_i^2}$, in particular, if $\ind(G)\leq 4$.
\end{theorem}
\begin{proof}{}
$\frac{\partial\inflp(x,w)}{\partial x}=(1-w^2)\frac{we^x}{(1+we^x)^2}
=(1-w^2)\frac{y}{(1+y)^2}$ for $y=we^x$. Since the maximum of
$\frac{y}{(1+y)^2}$ is $\frac{1}{4}$ (for $y=1$), we get
$\sup_{i=1,\ldots,n, s\in\R}\left(\frac{\partial\inflp(x,w)}{\partial
x}\big|_{(s,w_i)}\right)\leq\frac{1-\min_i w_i^2}{4}$.  Hence by
Thm.~\ref{thm:convergence-sum}, Euler-based semantics converges if
$\ind(G)<\frac{4}{1-\min_i w_i^2}$.
This covers also the case of $\ind(G)\leq 4$ in case that 
$w_i>0$ for all $i$. Since Euler-based semantics satisfies
\emph{stickiness-min}, all arguments
with initial weight $0$ stay at $0$ and hence have no influence
by \emph{stability} and \emph{directionality}. Hence, we can ignore
them.
\end{proof}

\begin{theorem}\label{thm:convergence-quadratic}
  Given $\langle {G, w}\rangle$, quadratic energy semantics converges if
  $\ind(G)\leq 1$.
\end{theorem}
\begin{proof}{}
By
Thm.~\ref{thm:convergence-sum}, since
$\sup_{i=1,\ldots,n,
  s\in\R}\left(\frac{\partial\mathit{qmax}(x,w)}{\partial
  x}\big|_{(s,w_i)}\right) \leq 0.65$.
\end{proof}

\begin{theorem}\label{thm:direct-aggregation-convergence}
  For $\ind(G)<\delta$, direct aggregation semantics converges; indeed,
  it converges to $(I-\frac{1}{\delta}G)^{-1}w$.
\end{theorem}

\begin{proof}{} 
Since $\frac{\partial \inflp(x,w_i)}{\partial x}=\frac{1}{\delta}$, convergence follows
already from
 Thm.~\ref{thm:convergence-sum}.
A more specific convergence proof uses
$$\rownorm{\frac{1}{\delta}G}= \frac{\ind(G)}{\delta} < 1$$
By  \cite{Horn:1985:MA:5509}, Corollary 5.6.16,
this implies that
$\sum_{i=0}^\infty \ (\frac{1}{\delta}G)^i=(I-\frac{1}{\delta}G)^{-1}$, hence
$$\accdegrvec{\semdir}{\arggraph,\delta}=\lim_{i\to\infty}f^i=\sum_{i=0}^\infty \ (\frac{1}{\delta}G)^iw=
(I-\frac{1}{\delta}G)^{-1}w\vspace*{-0.4cm}$$
\end{proof}

For semantics using \emph{top}, we can obtain stronger convergence
results.

\begin{lemma}\label{lem:top}
$$\rownorm{\mathit{top}(G,d_1)d_1-\mathit{top}(G,d_2)d_2}
\leq 2\rownorm{d_1-d_2}$$
\end{lemma}
\begin{proof}{}
This is shown by considering that for a node $j$,
the maximum support of $j$ in $d_1$ and
that in $d_2$ can differ by at most $\rownorm{d_1-d_2}$,
and likewise for attacks (hence the factor $2$).
\end{proof}{}

\begin{theorem}\label{thm:convergence-top}
 Assume that we use $\mathit{top}$ for aggregation $\aggrp$. Fix
  $\langle {G, w}\rangle$. If $$m=\sup_{i=1,\ldots,n,
  s\in\R}\left(\frac{\partial\inflp(x,w)}{\partial
  x}\big|_{(s,w_i)}\right)<\frac{1}{2},$$ then
  $\lim_{i\to\infty}f^i$ converges.
\end{theorem}
\begin{proof}{}
Analogous to the proof of Thm.~\ref{thm:convergence-sum}, but
replacing $\rownorm{G}$ with $2$, which is justified by
Lemma~\ref{lem:top}.
\end{proof}

\begin{theorem}\label{thm:max-Euler-convergence}
  Max Euler-based semantics converges.
\end{theorem}
\begin{proof}{}
By Thm.~\ref{thm:convergence-top} and the fact that $m\leq\frac{1}{4}$
obtained from the proof of Thm.~\ref{thm:convergence-euler}.
\end{proof}

\begin{theorem}\label{thm:conv-max-based-aggr}
Damped max-based semantics converges for $\delta>2$.
\end{theorem}
\begin{proof}{}
By Thm.~\ref{thm:convergence-top} and the fact that $m=\frac{1}{\delta}$.
\end{proof}

The results for linear $\inflp$ transfer to the sigmoid case,
because $\sigma$ is a continuous bijection with continuous inverse:

\begin{theorem}\label{thm:sigmoid-direct-aggregation-convergence}
  For $\delta>\ind(G)$, sigmoid direct aggregation semantics converges; indeed,
  it converges to $\sigma((I-\frac{1}{\delta}G)^{-1}\sigma^{-1}(w))$.
\end{theorem}
\begin{proof}{}
  Easy from Thm.~\ref{thm:direct-aggregation-convergence} by noticing
  that $\sigma$ is a continuous bijection with continuous inverse.
\end{proof}

\begin{theorem}\label{thm:sigmoid-conv-max-based-aggr}
Sigmoid damped max-based semantics converges for $\delta>2$.
\end{theorem}
\begin{proof}{}
  Easy from Thm.~\ref{thm:conv-max-based-aggr} in the same
  way as in the proof of Thm.~\ref{thm:sigmoid-direct-aggregation-convergence}.
\end{proof}

\section{Divergence}\label{sec:divergence}

A very simple example of divergence can be obtained for
combined aggregation-based/h-categorizer semantics:\\
  \begin{example}\label{ex:non-convergence-combined}
For the \wasa
  {
  \xymatrixcolsep{2pt}
  $\xymatrix{
  \frac{3}{4}&&a\ar@(l,u)\ar@{<-->}[rrrr] &&&&
  b \ar@(r,u)&&\frac{1}{4}\\
  }$,
  }
with combined aggregation-based/h-categorizer semantics,   we have $f^{2i}=\matr{\frac{3}{4}&\frac{1}{4}}^T$
  and $f^{2i+1}=\matr{\frac{1}{2}&\frac{1}{2}}^T$. Thus, the $f^i$
  do not converge.
  \end{example}

  We now prove a far more general divergence result for semantics
  based on $\mathit{sum}$.
\begin{theorem}\label{thm:divergence}
Assume $\aggrp=\mathit{sum}$. If there are numbers $v>w\in\D$
and $k\in\N$, $k\geq 1$ with
$$\inflp(k(v-w),w)\geq\inflp(k(w-v),v)$$
then there is a \wasa with maximal indegree
$2k$ that
leads to divergence.
\end{theorem}
\begin{proof}{}
Consider the \wasa\\[1ex]
  \xymatrixcolsep{2pt}
  $$\xymatrix{
  v&&a_1,\ldots, a_k\ar@(l,u)\ar@{<-->}[rrrr] &&&&
  b_1,\ldots, b_k \ar@(r,u)&&w\\
  }$$
i.e.\ all the $a_i$ attack each other, all the $b_i$ attack each
other, all $a_i$ support all $b_j$, and vice versa.
By \emph{stability}-$\iota$ and \emph{reinforcement-$\inflp$}, $f^1_{a_i}=\inflp(k(w-v),v)<v=f^0_{a_i}$
and $f^1_{b_i}=\inflp(k(v-w),w)>w=f^0_{b_i}$. Moreover, the assumption
can  be rewritten as $f^1_{b_i}\geq f^1_{a_i}$.
By induction over $j$, we simultaneously prove that
$$\begin{array}{lll}
f^{j}_{a_i}\leq f^{j+2}_{a_i}\textrm{ if $j$ is even}&&f^{j}_{a_i}\geq f^{j+2}_{a_i}\textrm{ if $j$ is odd}\\
f^{j}_{b_i}\geq f^{j+2}_{b_i}\textrm{ if $j$ is even}&&f^{j}_{b_i}\leq f^{j+2}_{b_i}\textrm{ if $j$ is odd}\\
\end{array}$$
Induction basis: we have $f^0_{a_i}=v=\inflp(0,v)$ using \emph{stability}-$\iota$.  By
\emph{reinforcement-}$\inflp$, from $f^1_{b_i}\geq f^1_{a_i}$ we get
$\inflp(0,v)\leq \inflp(k(f^1_{b_i}-f^1_{a_i}),v)=f^2_{a_i}$.  Also,
$f^0_{b_i}=w=\inflp(0,w)\geq\inflp(k(f^1_{a_i}-f^1_{b_i}),w)=f^2_{b_i}$.  Induction step: assume that $j+1$ is odd. Then
$f^{j+1}_{a_i}=\inflp(Gf^{j},w_{a_i})=\inflp(k(f^{j}_{b_i}-f^{j}_{a_i}),w_{a_i})$. By
induction hypothesis and \emph{reinforcement-$\inflp$}, this is
$\geq\inflp(k(f^{j+2}_{b_i}-f^{j+2}_{a_i}),w_{a_i})=\inflp(Gf^{j+2},w_{a_i})=f^{j+3}_{a_i}$. Similarly for the case that $j+1$ is even, and for
 $b_i$.

Altogether, we have
$$\begin{array}{l}
\ldots\leq f^5_{a_i}\leq f^3_{a_i}\leq f^1_{a_i}<f^{0}_{a_i}\leq f^{2}_{a_i}\leq f^{4}_{a_i}\leq\ldots\\
\ldots\geq f^5_{b_i}\geq f^3_{b_i}\geq f^1_{b_i}>f^{0}_{b_i}\geq f^{2}_{b_i}\geq f^{4}_{b_i}\geq\ldots
\end{array}$$
which means that the sequence $f^j$ diverges.
\end{proof}

\begin{example}\label{ex:divergence}
\label{ex:more-non-convergence}
Thm.~\ref{thm:divergence} can  be applied as follows:
\def\arraystretch{1.6}
\begin{tableenv}
\begin{tabular}{|l|c|c|c|}\hline
$\inflp$ & k & v & w\\\hline
multilinear & 1& $\largefrac{1}{2}$& $\largefrac{2}{5}$\\\hline
QMax & 1& $1$& $0$\\\hline
combined fractional & 2& $\largefrac{1}{2}$& $\largefrac{2}{5}$\\\hline
Euler-based & 3& $\largefrac{1}{2}$& $\largefrac{2}{5}$\\\hline
linear($\delta$) & $\largefrac{\delta'}{2}$& $\largefrac{2}{3}$& $\largefrac{3}{5}$\\\hline
sigmoid($\delta$) & $\largefrac{\delta'}{2}$& $\largefrac{2}{3}$& $\largefrac{3}{5}$\\\hline
\end{tabular}
\end{tableenv}
Here, $\delta'=\begin{cases}
\delta+2&\textrm{ if $\delta$ is even}\\
\delta+1&\textrm{ if $\delta$ is odd}\\
\end{cases}$

In particular, there is a \wasa with maximal indegree of 6 for which
Euler-based semantics diverges, and for each $\delta$, there is a
\wasa with maximal indegree of $\delta+1$ (for odd $\delta$) or $\delta+2$ (for even $\delta$)
for which (sigmoid) direct
aggregation semantics with parameter $\delta$ diverges.\qed







\end{example}

\begin{example}\label{ex:divergence-direct-aggregation}
Given $\delta\in\N$ even, there is a \wasa with maximal indegree of $\delta$ for
which (sigmoid) direct aggregation semantics with parameter $\delta$
diverges. Let $k=\frac{\delta}{2}$, use the graph from
Thm.~\ref{thm:divergence} and initial weights $a_i=\frac{3}{4}$
and $b_i=\frac{3}{4}$. Then the behaviour is like that in
Ex.~\ref{ex:non-convergence-combined}.
\end{example}

Indeed, there is a deeper reason for the many counterexamples
in Ex.~\ref{ex:divergence}:
\begin{theorem}\label{thm:divergence-prerequisites}
For any function $\inflp$ satisfying the essential characteristics,
there are numbers $v>w\in\D$
and $k\in\N$, $k\geq 1$ with
$$\inflp(k(v-w),w)\geq\inflp(k(w-v),v)$$
\end{theorem}
\begin{proof}{}
Choose $w\in\D$ such that if $\maxd$ exists, then $w\neq\maxd$.
Choose $\varepsilon>0$ such that $w+\varepsilon\in\D$. Choose
$k\in\N$ such that $k\geq\largefrac{\varepsilon}{\inflp(\varepsilon,w)-w}\geq 0$ (the latter inequality holds due to \emph{reinforcement-$\inflp$} and \emph{stability}-$\iota$). 
Let $v=w+\largefrac{\varepsilon}{k}$. This is in $\D$ since $w\in\D$, $w+\varepsilon\in\D$ and $\D$ is connected.  Due to \emph{reinforcement-$\inflp$} and \emph{stability}-$\iota$,
$\inflp(k(w-v),v)\leq \inflp(0,v) = v = w+\largefrac{\varepsilon}{k}\leq
 w+\varepsilon \largefrac{\inflp(\varepsilon,w)-w}{\varepsilon} = \inflp(\varepsilon,w) = \inflp(k(v-w),w)$.
\end{proof}

Combining Thm.~\ref{thm:divergence} and
Thm.~\ref{thm:divergence-prerequisites}, we obtain:
\begin{theorem}
There is no well-behaved modular semantics based on \emph{sum} that converges for
all \wasa\!\!s.
\end{theorem}


\section{Related Work}\label{sec:related}

Argumentation graphs that assign real numbers as weights to arguments
have been widely discussed, in particular in
\cite{Cayrolc05,DBLP:conf/ictai/CayrolL11,DBLP:conf/sum/AmgoudB13a},
and also for the bipolar case (involving both attack and support
relations) \cite{DBLP:conf/ecsqaru/CayrolL05,DBLP:conf/iberamia/BudanVS14}. However, these works do
not consider initial weightings.  The latter are considered mainly in
\cite{DBLP:conf/ijcai/AmgoudB16} (for support relationships only) and
in \cite{DBLP:conf/ijcai/AmgoudBDV17} (for attack relationships only),
see also \cite{DBLP:journals/argcom/BaroniRTAB15} for a different approach.

The need for bipolar argumentation graphs has been empirically supported
in \cite{DBLP:journals/ijar/PolbergH18}.
Our central reference for bipolar weighted argumentation graphs with
initial weightings is \cite{DBLP:conf/ecsqaru/AmgoudB17}.  We are in
particular building on the notions and results developed there,
including their \emph{Euler-based semantics}. With our framework, we
can explore the borderline between convergence and divergence of
Euler-based semantics, and propose a variant of Euler-based semantics
that always converges.

We propose a novel modular approach to bipolar weighted argumentation.
There are some other modular approaches to argumentation in the
literature, where the valuation of arguments is obtained as modular
composition of some functions, notably social abstract argumentation
\cite{DBLP:conf/ijcai/LeiteM11,DBLP:conf/tafa/EgilmezML13}
and the local valuation approach
of \cite{DBLP:journals/ijis/AmgoudCLL08}. These however do neither
match our topic of weighted argumentation as studied
in \cite{DBLP:conf/kr/AmgoudB16,DBLP:conf/ijcai/AmgoudB16,DBLP:conf/ecsqaru/AmgoudB17,DBLP:conf/ijcai/AmgoudBDV17}
and captured by our Def.~\ref{def:semantics}.  Nor do these modular
approaches lead to general convergence and divergence results.

A study of characteristics of (bipolar weighted) argumentation
frameworks has been given in \cite{DBLP:conf/ijcai/AmgoudB16}.
Characteristics have been systematically grouped in
\cite{DBLP:conf/aaai/BaroniRT18}.  Our modular approach leads to a
more orthogonal formulation of characteristics, since they can
be split in properties of aggregation $\aggrp$ and if
influence $\inflp$.

We use matrix notation for argumentation graphs. Such a notation has
been used e.g.\ \cite{DBLP:conf/tafa/XuC15} in order to prove
characterisations of different types of extensions (stable, complete,
\ldots) for a Dung-style framework with attacks only. In
\cite{DBLP:conf/comma/CoreaT16}, matrices for bipolar graphs have been
introduced, and matrix exponentials have been used for characterising
weighted paths in argumentation graphs. This resembles the use of a
limit of matrix powers in the proof of convergence of direct aggregation semantics (Thm.~\ref{thm:direct-aggregation-convergence}). We also build a semantics
based on matrix exponentials. While it always converges, it does
not satisfy the \emph{equivalence} characteristics, a crucial
locality principle.
Convergence for direct aggregation semantics can be proved for bounded indegree,
using the maximum row sum norm,
which coincides with the maximal indegree of the argumentation graph.
Also for other semantics like Euler-based semantics, matrix norms
play a crucial role for proving convergence.
Note that to our knowledge, matrix norms have not been used for
argumentation graphs before.

For the specification of the resulting degrees of arguments, our
approach uses a discrete iteration in order to reach a fixpoint (see
the limit formula (\ref{eq:limit}) in Thm.~\ref{thm:limit}), following
and generalising the approaches of
\cite{DBLP:conf/kr/AmgoudB16,DBLP:conf/ecsqaru/AmgoudB17,DBLP:conf/ijcai/AmgoudBDV17}.
Using differential equations, \cite{Potyka18} introduces a continuous
version of this discrete iteration for one specific semantics, the
quadratic energy model.  While the continuous approximation of the
degree seems to converge more often than the discrete version, in
\cite{Potyka18} only a rather weak convergence result is proved,
namely for acyclic graphs.  By contrast, we prove convergence results
for cyclic graphs as well.  We have included a discrete version of
Potyka's quadratic energy model in our overview above.

Our work has stimulated new research: \cite{Potyka18b} uses our
modular approach, acknowledging that it leads to easier verification
of characteristics. In \cite{Potyka18b}, our
convergence results are generalised using Lipschitz continuity.
Actually, our use of a maximal derivative of the influence function
$\inflp$ can be replaced by the Lipschitz constant of
$\inflp$. Lipschitz constants cannot only be provided for $\inflp$,
but also for the aggregation function $\aggrp$. In \cite{Potyka18b}, a
new (optional) characteristics \emph{duality} is introduced in terms
of $\aggrp$ and $\inflp$; it requires a certain symmetry (w.r.t.\ the
mean value of $\D$) in the behaviour of supports and attacks.
Note that \emph{duality} is different from and complementary to
our \emph{symmetry}. \cite{Potyka18b} introduces a further
aggregation function $\aggrp$, product-aggregation, which captures the DF-QuAD algorithm
\cite{DBLP:conf/kr/RagoTAB16}.
\cite{Potyka18b} also generalises our modular semantics to the
continuous case and provides a continuous version of our limit formula
(\ref{eq:limit}), again using differential equations. The latter can
be approximately solved using e.g.\ Euler's method, and then our limit
formula (\ref{eq:limit}) turns out to be the special case where a
step-size of 1 is used. This means that the continuous version will
converge in more cases than our discrete version.  However, no general
convergence results are known for the continuous case.

\section{Conclusion and Future work} \label{sec:future}

In this work we have focussed on well-behaved modular acceptability semantics
for bipolar weighted argumentation graphs. The semantics are modular in the sense
that the acceptability degree of an argument is calculated in two steps: firstly, an 
\emph{aggregation} function $\aggrp$ combines the effect of the predecessors of an argument; secondly, the 
\emph{influence} function $\inflp$ calculates the acceptability degree of the argument based on the result of the aggregation and the initial weights. Well-behaved modular acceptability semantics are defined 
based on a set of structural and essential characteristics. 
Many of these characteristics have already been studied in the literature. 
Our modular matrix-based approach
allowed us to axiomatise them in a mathematically elegant way
 that links the characteristics to 
requirements on aggregation function and the influence function.

For the aggregation function $\aggrp$, we discuss several alternatives from the
literature, including \emph{sum}, which sums up all the weights of parent
arguments, \emph{top}, which only consider the strongest supporter and
the strongest attacker, and \emph{reward} and \emph{card}, which
favour number of arguments over the quality. We also discuss eight different 
implementation of the influence functions. All implementations
of $\aggrp$ and $\inflp$ satisfy the axiomatisation of the essential characteristics,
although some of them only under restricting assumptions. 
Five implementations of the influence function and five implementations of the aggregation function 
satisfy the essential characteristics unconditionally.  
These can be modularly
combined, leading to different acceptability semantics. We discussed a selection of 
these acceptability semantics, and illustrate their differences with the help of an example.

Our modular matrix-based approach
 simplifies the study the 
convergence of semantics, since a semantics is built compositionally of
two orthogonal parts. In this general setting,
we can already prove that all well-behaved modular acceptability semantics converge for all
acyclic graphs.

In this work,
have have concentrated on acceptability semantics based on \emph{sum} and \emph{top}.
We show that  no well-behaved modular acceptability semantics based on
\emph{sum} converges for all graphs. We also give specific
counterexamples for Euler-based semantics, direct aggregation
semantics and sigmoid direct aggregation semantics. We provide
convergence and divergence theorems that explore the boundary between
convergence and divergence for \emph{sum}-based
semantics. Instantiating these theorems, we can show that Euler-based
semantics converges for cyclic graphs of indegree at most 4 and in
general does not converge for a maximal indegree of 6.  For (sigmoid)
direct aggregation semantics (which is based on a parameter $\delta$),
convergence holds for graphs with indegree smaller than $\delta$, and we
can show divergence for indegree $\delta$ (for even $\delta$) or $\delta+1$
(for odd $\delta$).

The situation is much better for \emph{top}.  It does not exhibit the
problem with large indegrees: since only the top supporter and
attacker are considered, the indegree is irrelevant here.  We show
that three max-based semantics built on \emph{top} converge. These are
the first semantics that converge for all bipolar weighted
argumentation graphs.

It has been argued that \emph{top}, i.e., consideration of only
the strongest attacker and supporter, is an unintuitive principle.
In general, our goal is to show that our framework enables the definition and
study of convergence/divergence properties for a wide variety of semantics that
cover different intuitions. Hence we do not want to exclude \emph{top},
if it is supported by some intuitions.
Indeed, \cite{DBLP:conf/ijcai/AmgoudB16} argue that
there are different intuitions about whether the cardinality of
arguments or their quality should be given preference. They suggest
the top-based semantics as implementation of the quality preferences.
\cite{DBLP:conf/ijcai/AmgoudBDV17} argue:
\begin{quote}
``If a Fields medal[ist] says $P$, whilst
three students say $\neg P$, we probably believe $P$. [...] this principle
is similar to the Pessimistic rule in decision under uncertainty
[Dubois and Prade, 1995].''
\end{quote}
This is applied by Wikipedia, where expert
arguments are regularly given precedence over a multitude of arguments
by less informed contributors. 

Still, some of the literature seems to favour \emph{sum} over other
aggregation functions. Our divergence results show that there is no
\emph{sum}-based semantics that converges for all graphs. The best
compromise seems then to be the use of (sigmoid) direct aggregation
semantics, while choosing the parameter $\delta$ large enough such it will
be greater than the indegrees of the considered argumentation graphs.


\medskip
A possible direction of
future research is to equip attack and support relations with weights,
e.g.\ in the interval $[-1,1]$.  Our matrix-based approach greatly
eases working out such an approach.  Indeed, due to the use of a matrix
norm, many of our theorems already generalise to this case.  See
\cite{DBLP:journals/ai/DunneHMPW11,DBLP:conf/kr/Coste-MarquisKMO12,DBLP:conf/aiia/PazienzaFE17a}
for work in this direction, but in a different context: only attacks
and supports are equipped with weights, not the arguments themselves.

Also the study of characteristics leaves some open questions. For
example, is it possible to generalise \emph{counting} in a way that one does
not consider exactly the same set of attackers, but a set of
comparable attackers?

Moreover, \cite{Potyka18b} provides a continuous version of the limit
formula (\ref{eq:limit}) in Thm.~\ref{thm:limit}. An important open
question is convergence for the continuous case, but general convergence 
results seem much harder to obtain here.

Also, we would like to use our framework to define a semantics for
the Argument Interchange Format (AIF, \cite{RahwanReed09}) that is
simpler and more direct than the one given in the literature
\cite{DBLP:journals/logcom/BexMPR13}.

Finally, large argumentation graphs will benefit from a modular
design; e.g.\ in \cite{BetzCacean2012} they are often divided into
subgraphs, e.g.\ by drawing boxes around some groups of arguments.
The characteristics of our semantics suggest that modularity
can be obtained by substituting suitable subgraphs with discrete
graphs whose arguments are initially weighted with their degrees
in the original graph.

\bibliographystyle{plain}
\bibliography{refs}

\newpage
\appendix
\section{Proofs}
\noindent\textbf{Theorem~\ref{thm:char-independence}}
All essential axioms are independent of each other.

\begin{proof}{}
It suffices to prove
independence for the aggregation and influence axioms separately,
since these sets of axioms are orthogonal.

We prove independence of the influence axioms by giving, for each
axiom, an implementation based on $\D=[-1,1]$
that falsifies the axiom while satisfying
all the others.
For \emph{reinforcement-$\inflp$}, use $\inflp(s,w)=w$.
For \emph{initial monotonicity}, use $\inflp(s,w)=w(s+1)$.
For \emph{stability-$\inflp$}, use $\inflp(s,w)=s+w+1$.
For \emph{continuity-$\inflp$}, use $$\inflp(s,w)=\begin{cases}
w & s=0\\
s+w-1 & s<0\\
s+w+1 & s>0\\
\end{cases}$$

We now come to the axioms for $\aggrp$, using the same method.

\emph{Anonymity-2}: use $\aggrp(g,d)=gd\cdot \begin{cases}1&k=0\\0&otherwise\end{cases}$, where $$k=\sum_{i=1,\ldots,n,g_i<0,d_i<0} d_i\cdot |\{j\in\{1\ldots i\}\mid g_j<0, d_j<0\}|.$$

\emph{Independence}: use $\aggrp(g,d)=\largefrac{gd}{n}$.

\emph{Reinforcement-$\aggrp$}: use $\aggrp(g,d)=g\cdot\abs(d)$,
where $\abs$ is taken entrywise.

\emph{Parent monotonicity-$\aggrp$}: use $\aggrp(g,d)=gd\cdot |k|$ where
$$k=\sum_{i=1,\ldots,n,d_i\neq 0}g_i.$$

\emph{Stability-$\aggrp$}: use $\aggrp(g,d)=gd+k$ where
$$k=\sum_{i=1,\ldots,n,d_i< 0}d_i.$$

\emph{Continuity-$\aggrp$}: use $\aggrp(g,d)=\begin{cases}gd-1&gd<0\wedge\exists i\in\{1\ldots n\}.d_i<0\\gd&otherwise\end{cases}$.

\emph{Neutrality}:  use $\aggrp(g,d)=gd'$ where
$d'_i=d_i+1$.

\emph{Strengthening-$\aggrp$}, \emph{weakening-$\aggrp$}: use
$\aggrp(g,d)=0$.  Note that this only proves the disjunction of
\emph{strengthening-$\aggrp$} and \emph{weakening-$\aggrp$} to be
independent from the remaining essential axioms. However, it is
clear that the remaining essential axioms do not entail the symmetry
conditions that are needed to prove \emph{strengthening-$\aggrp$}
from \emph{weakening-$\aggrp$} or vice versa. Such symmetry
conditions would be \emph{symmetry} together with \emph{duality}-$\aggrp$
from \cite{Potyka18b}.


\end{proof}

\noindent\textbf{Theorem~\ref{thm:char-entailment}}
\emph{Reinforcement-$\inflp$}  entails  \emph{parent monotonicity-$\inflp$}.
\emph{Stability} entails \emph{soundness}.
\emph{Reinforcement-$\inflp$} and \emph{stability} together entail  \emph{strengthening-$\inflp$}
and \emph{weakening-$\inflp$}.
\emph{Reinforcement-$\aggrp$}  entails  \emph{directionality}.

\begin{proof}{}
  Mostly straightforward. For \emph{directionality}, note that $d^1\equiv_g d^2$
  implies $d^1\leq_g d^2$ and $d^2\leq_g d^1$.
\end{proof}

\noindent\textbf{Theorem~\ref{thm:French-characteristics}}
  Any well-behaved modular acceptability semantics \bs (i.e., a semantics that satisfies our structural and essential characteristics
  in Table~\ref{tab:characteristics} above) satisfies the following
  characteristics
  defined in \cite{DBLP:conf/ecsqaru/AmgoudB17} (with
$0$ replaced by $\mind$ and 1 by $\maxd$):
anonymity,
bi-variate independence,
bi-variate equivalence,
bi-variate directionality (under the assumption of
convergence as in Def.~\ref{def:conv}),
stability,
neutrality,
the parent monotonicity part of bi-variate monotony
non-strict bi-variate reinforcement,
weakening, strengthening
and resilience.

\begin{proof}{}
We need some notation from \cite{DBLP:conf/ecsqaru/AmgoudB17}.
$\attackers{}{a_i}$ is the set of all attackers of $a_i$ in \arggraph,
that is $\attackers{}{a_i} = \{ a_j | \ g_{ij} = -1\} $, and
$\supporters{}{a_i}$ is the set of all supporters of $a_i$ in
\arggraph, that is $\supporters{}{a_i} = \{ a_j | \ g_{ij} = 1\}$.
Moreover, $\sattackers{}{a_i}$ is the the set of \emph{significant}
attackers of $a_i$, i.e.\ attackers $a$ with
$\accdegr{}{}{a}\neq 0$. Likewise, $\ssupporters{}{a_i}$ is the
the set of significant supporters of $a_i$.

\emph{Anonymity} is stated in \cite{DBLP:conf/ecsqaru/AmgoudB17} as follows: for any two \wasa 
	\defaultarggraph and \alternativearggraph and 	
	for any isomorphism $f$ from \arggraph to $\arggraph^\prime$, the following property holds: 
	for any $ a$  in \argset,  
	$ \accdegr{}{}{a}= \accdegr{}{\arggraph^\prime}{f(a)}$.
This
  follows from our \emph{anonymity-2} with Lemma~\ref{lem:anonymity-simplified}, since any graph
isomorphism induces a permutation matrix transforming the incidence
matrix of the source graph into that of the target graph.

\emph{Bi-variate independence} is stated in \cite{DBLP:conf/ecsqaru/AmgoudB17} as follows:  for any two \wasa 
	\defaultarggraph and \alternativearggraph such that 
    $\argset $ and $ \argset ^\prime$ do not share a component, the following property 
	holds: for any $a$ in $\argset$, $\accdegr{}{}{a} = \accdegr{}{\arggraph \oplus \arggraph 
	 ^\prime}{a}$. 
This property easily follows from
$$G=\matr{G_1&0\\0&G_2}\wedge w=\matr{w_1\\w_2}\to\accdegrvecnew{}{}=\matr{\accdegrvecnew{}{\langle {G_1,
       w_1}\rangle}\\\accdegrvecnew{}{\langle {G_2, w_2}\rangle}}$$
which in turn follows from repeated application of
our \emph{independence}, using equation~(\ref{eq:rec2}).

\emph{Bi-variate equivalence} is stated in \cite{DBLP:conf/ecsqaru/AmgoudB17} as follows:
 for any weigh\-ted argumentation graph 
\defaultarggraph{} and for any $a, b$ in $\argset $, if 
	\begin{itemize}
		\item $w(a) = w(b)$, 
		\item there exists a bijective function $f$ from  \attackers{}{a} to \attackers{}{b} 
		such that $\forall x \in \attackers{}{a}$, $\accdegr{}{}{x}= \accdegr{}{}{f(x)}$, 
		\item there exists a bijective function $g$ from  \supporters{}{a} to \supporters{}{b} 
		such that $\forall x \in \supporters{}{a}$, $\accdegr{}{}{x}= \accdegr{}{}{g(x)}$, 
	\end{itemize}
		then $\accdegr{}{}{a} = \accdegr{}{}{b}  $. 	
This follows using Thm.~\ref{thm:equivalence}, noting
that equation~(\ref{eq:modularity}) turns equation~(\ref{eq:rec2})
into equation~(\ref{eq:rec0}) and \emph{anonymity-2} into
equation~(\ref{eq:deg}).

\emph{Bi-variate directionality} is stated in \cite{DBLP:conf/ecsqaru/AmgoudB17} as follows: for any two \wasa 
\defaultarggraph and \alternativearggraph  such that $\argset=\argset'$,
$w=w'$\footnote{Actually this condition is missing in \cite{DBLP:conf/ecsqaru/AmgoudB17}.},
$G'$ is obtained from $G$ by adding an attack or support from $a$ to $b$,
and there is no path from $b$ to $x$, we have
$\accdegr{}{}{x} = \accdegr{}{\arggraph^\prime}{x}$.
To prove this, we need to assume
that the limit in equation~(\ref{eq:limit}) converges. Assume that $G'$ is
obtained from $G$ by adding a new edge from $a$ to $b$ (either support
or attack), and there is no path from $b$ to $x$. Let $C$ be the set
of ancestors of $x$, hence $b\not\in C$.  Then by induction over $i$,
we can prove that $(f^i_{(G,w)})_c=(f^i_{(G',w)})_c$ for all $c\in
C$. The induction base follows from $w=w'$. Concerning the inductive step, let
$c\in C$ by arbitrary. From the inductive hypothesis, we get
$f^i_{(G,w)}\equiv_{G_c}f^i_{(G,w')}$ because $b\not\in C$.  By
applying our \emph{directionality}, we obtain
$\aggrp(G_c,f^i_{(G,w)})=\aggrp(G_c,f^i_{(G,w')})$. Then
$(f^{i+1}_{(G,w)})_c=\inflp(\aggrp(G_c,f^i_{(G,w)}),w_c)=\inflp(\aggrp(G_c,f^i_{(G,w')}),w_c)=(f^{i+1}_{(G',w)})_c$.
Hence $(f^i_{(G,w)})_c=(f^i_{(G',w)})_c$ for all $c\in
C$. By Thm.~\ref{thm:limit}, 
$\accdegr{}{}{c} = \accdegr{}{\arggraph^\prime}{c}$ for all $c\in C$,
in particular for $c=x$.

\emph{Stability} is stated in \cite{DBLP:conf/ecsqaru/AmgoudB17} as
follows: for any \wasa \defaultarggraph, for any argument $a$ in $
\argset$, if $\attackers{}{a} = \supporters{}{a} = \emptyset$, then
$\accdegr{}{}{a} = w(a)$. But if the set of attackers and supporters
of an argument $a$ is empty, the corresponding row in $G$ is $0$. By
our two \emph{stability} axioms,
$\accdegr{}{}{a} =\inflp(\aggrp(0,\accdegr{}{}{}),w_a)=\inflp(0,w_a)=w_a$.

\emph{Neutrality} is stated in \cite{DBLP:conf/ecsqaru/AmgoudB17} as
follows: for any \wasa \defaultarggraph , for all $a, b\in \argset$,
if $w(a) = w(b)$, $\attackers{}{a} \subseteq \attackers{}{b}$
$\supporters{}{a} \subseteq \supporters{}{b} $,
$\attackers{}{b}\cup\supporters{}{b} =
\attackers{}{a}\cup\supporters{}{a} \cup \{x\}$ and $\accdegr{}{}{x} =
0$, then $\accdegr{}{}{a} = \accdegr{}{}{b}$.  This follows directly
from our \emph{neutrality}. The argument with number $k$ is the additional argument
$x$ with neutral initial weight $0$.

The \emph{parent monotonicity} part\footnote{Note that we consider the
  \emph{counting} part separately, because it is optional.} of
\emph{bi-variate monotony} in \cite{DBLP:conf/ecsqaru/AmgoudB17} can
be stated as follows: for any  \wasa \defaultarggraph  and any arguments $a$
$b$, if
$w(a) = w(b)\geq 0$\footnote{Note that we do not need the stronger condition
  $w(a)>0$ from \cite{DBLP:conf/ecsqaru/AmgoudB17} here.}, $\attackers{}{a} \subseteq \attackers{}{b}$ and
$\supporters{}{b} \subseteq \supporters{}{a} $,
	then $\accdegr{}{}{b}  \leq \accdegr{}{}{a} $.
In order to show this, 
assume that the supporters of $b$ are included in those
of $a$ and the attackers of $a$ included in those of $b$. In our
matrix notation, this means that $G_b\leq G_a$. Moreover, assume that
$w_a=w_b\geq 0$.
Let $D=\accdegr{}{}{}$.
By our \emph{parent monotonicity-$\aggrp$},
$\aggrp(G_b,D)\leq\aggrp(G_a,D)$. By our \emph{parent
monotonicity-$\inflp$},
$D_b=\inflp(\aggrp(G_b,D),w_b)=\inflp(\aggrp(G_b,D),w_a)\leq\inflp(\aggrp(G_a,D),w_a)=D_a$.

Non-strict \emph{bi-variate reinforcement} is stated in
\cite{DBLP:conf/ecsqaru/AmgoudB17} as follows: for any  \wasa
\defaultarggraph, $C,C'\subseteq\argset$, arguments $a,b\in\argset$
and $x,x',y,y'\in\argset\setminus(C\cup C')$ such that:
\begin{enumerate}
\item $w(a) = w(b)$\footnote{The condition $w(a)>0$ from \cite{DBLP:conf/ecsqaru/AmgoudB17} is not needed for
  the non-strict version.}
  \item $\accdegr{}{}{x} \leq \accdegr{}{}{y}$,
  \item $\accdegr{}{}{x'} \geq \accdegr{}{}{y'}$,
  \item $\attackers{}{a} = C\cup\{x\}$
  \item $\attackers{}{b} = C\cup\{y\}$
  \item $\supporters{}{a} = C'\cup\{x'\}$
  \item $\supporters{}{b} = C'\cup\{y'\}$
\end{enumerate}
we have $\accdegr{}{}{a} \geq \accdegr{}{}{b}$.
Assume the above conditions. Let $P$ be the permutation matrix
that exchanges $x$ with $y$, and $x'$ with $y'$. Then conditions
2 and 3 mean that $D\geq_{G_a}PD$, where $D=\accdegr{}{}{}$.
Conditions 4--7 mean that $G_a=G_bP^{-1}$. 
By our \emph{reinforcement-$\aggrp$}, $\aggrp(G_a,D)\geq\aggrp(G_a,PD)$.
But $\aggrp(G_a,PD)=\aggrp(G_bP^{-1},PD)=\aggrp(G_b,D)$ by \emph{anonymity-2}.
Using  \emph{reinforcement-$\inflp$}, we then get
$D_a=\iota(\aggrp(G_a,D),w(a))=\iota(\aggrp(G_a,D),w(b))\geq \iota(\aggrp(G_b,D),w(b))=D_b$.

Note that \emph{top} violates
strict \emph{bi-variate reinforcement}; therefore we do not impose an
axiom that would entail it.

\emph{Resilience} is stated in
\cite{DBLP:conf/ecsqaru/AmgoudB17} as follows:
if $\mind<w(a)<\maxd$, then $\mind<\accdegr{}{}{a}<\maxd$.
This easily follows from our \emph{resilience}.

\emph{Weakening} is stated in
\cite{DBLP:conf/ecsqaru/AmgoudB17} as follows: for any \wasa
\defaultarggraph, for all $a\in\argset$, if $w(a) > \mind$ and there exists an injective function $f$ from
$\supporters{}{a}$ to $\attackers{}{a}$ such that:
\begin{itemize}
\item for all $x\in\supporters{}{a}$, $\accdegr{}{}{x}\leq \accdegr{}{}{f(x)}$, and
\item either $\sattackers{}{a}\setminus\{f(x) \mid x \in \supporters{}{a}\} \neq\emptyset$ 
  or there is some $x \in \supporters{}{a}$ such that $\accdegr{}{}{x}< \accdegr{}{}{f(x)}$,
\end{itemize}
then $\accdegr{}{}{a} < w(a)$.
Under the above assumptions, we can extend $f$ to bijection on all arguments.
Let $P$ be the corresponding permutation matrix. The condition that
 $f$ maps
$\supporters{}{a}$ into $\attackers{}{a}$
means $\atts(G_a,D)\leq \supps(PG_a,PD)$, where $D=\accdegr{}{}{}$.
The first assumption above means
that $\atts(G_a,D)\circ D\leq \supps(PG_a,PD)\circ PD$.
The second assumption means that either the first or the second above
inequality is strict.
With our \emph{weakening}-$\aggrp$, we now can
infer that $\aggrp(G_a,\accdegrvecnew{}{})<0$. With our \emph{weakening}-$\inflp$, we
get that $\inflp(\aggrp(G_a,\accdegrvecnew{}{}),w(a))\prec w(a)$.
Since  $w(a) > \mind$, even $\inflp(\aggrp(G_a,\accdegrvecnew{}{}),w(a))< w(a)$.
Hence by
equation~(\ref{eq:rec2}), $\accdegr{}{}{a}<w(a)$.

\emph{Strengthening} is dual to \emph{weakening}.

\end{proof}

\noindent\textbf{Theorem~\ref{thm:aggr-char}}
  If \emph{reward} is restricted to graphs with support relations
  only, \emph{card} is restricted to graphs with attack relations only,
  and both are restricted to the domain $[0,1]$, then
  the different implementations of $\aggrp$ satisfy all structural and essential
  characteristics as shown in Table.~\ref{tab:aggr-impl}.
  
\begin{proof}{}
\emph{Anonymity-2}: 
For \emph{sum}, we have $\aggrp(gP^{-1},Pd)=gP^{-1}Pd=gd=\aggrp(g,d)$.
The other semantics use matrix multiplication in a more complicated
way, but it is clear that this remains isomorphism-invariant.

\emph{Independence} follows since matrix multiplication ignores zeros.

\emph{Reinforcement-$\aggrp$}:
If $\forall i=1\ldots n.g_id^1_i\leq g_id^2_i$,
then also $gd^1=\sum_{i=1\ldots n}g_id^1_i\leq\sum_{i=1\ldots n} g_id^2_i=gd^2$.
This easily carries over to the other aggregation functions.

\emph{Parent monotonicity-$\aggrp$}: \emph{sum}, \emph{sum\textit{-}pos} and \emph{top}:
this is an easy property of the scalar product.
For  \emph{card}, this means $s_1\leq s_2$ and
$n_1\leq n_2$. Now the formulas for these two aggregations
are monotonic in both arguments, if one takes into account
that  $n=0$ implies $s_i=0$. Moreover, if the graph contains
attacks only, then also $|n_1|\geq |n_2|$, hence
$\frac{1}{|n_1|}\leq \frac{1}{|n_2|}$ and thus
$\aggrp^{\mathit{card}}(g_1,d)\leq \aggrp^{\mathit{card}}(g_2,d)$.
Concerning  \emph{reward}, parent monotonicity has been proved already
in \cite{DBLP:conf/ijcai/AmgoudBDV17} (under the name of monotony).

\emph{Stability-$\aggrp$}: for \emph{sum} and \emph{top}, we have $0d=0$.
For \emph{reward} and \emph{card}, $g=0$ implies $n=s=0$, hence
the result is $0$.

\emph{Continuity-$\aggrp$}: matrix-vector multiplication is continuous
in the vector. Moreover, the formulas for \emph{reward} and \emph{card}
are continuous in $s$.

\emph{Neutrality}: The only possible difference of $G$ and $G'$ is in
row $k$.  Since $d_k=0$, we obtain $Gd=G'd$.  For \emph{reward} and
\emph{card}, we additionally need to use that only founded attackers
and supporters count.

\emph{Directionality}: Since $d$ and $d'$ can only differ for indices
where $g$ is $0$, we get that $gd=gd'$. For \emph{reward} and \emph{card},
this means $n=n'$ and $s=s'$.

\emph{Strengthening}-$\aggrp$: 
The assumptions mean that all negative
entries in $g$ have corresponding positive entries in $g$ with a
larger or equal value (for the same index) in $d$. Thus, in the sum of
the scalar product, the positive entries in $g$ can counterbalance the
negative entries. Since also $d\geq 0$, we get $gd\geq 0$.  This leads
to a positive result also for \emph{reward} and \emph{card}.
For \emph{top}, the counterbalancing needs to consider only the
strongest attack and support.

\emph{Weakening}-$\aggrp$: dual to \emph{Strengthening}-$\aggrp$.

\emph{Franklin}: it is clear that matrix multiplication leads to
attacks and supports canceling out each other.

\emph{Counting}: it is clear that matrix multiplication leads to
additional support (attack) increasing (decreasing) the acceptability more.

The proofs for the sigmoid semantics are similar. Here, we need to use
that $\sigma^{-1}(0)=0$.
\end{proof}

\noindent\textbf{Theorem~\ref{thm:in-char}}
All implementations of $\inflp$ shown in
Table~\ref{tab:aggr-impl} satisfy all the structural and essential characteristics for $\inflp$
shown in Table~\ref{tab:characteristics},
if the listed conditions on $s$ are respected.

\begin{proof}{}
\emph{Reinforcement-$\inflp$} is easy to see in most
cases. For Euler-based, it follows from the double anti-monotonic
position of $s$.

\emph{Initial monotonicity} also is easy to see in most
cases. Rewrite the multilinear $\inflp(s,w)$ to $s+(1-s)w$.

\emph{Stability-$\inflp$} is also easy.

\emph{Continuity} follows since only
continuous functions are combined, and noting that for combined
fractional, both formulas agree on $0$.

For those semantics where it holds, \emph{compactness} is clear.
Likewise, \emph{resilience}, \emph{stickiness-min} and \emph{stickiness-max} are easy to see.
\end{proof}

\end{document}